\documentclass[letterpaper, 12pt]{article}

\usepackage{amssymb,amsmath}
\usepackage{epsf}
\usepackage{times,bm,mathrsfs}
\usepackage{amsmath}
\usepackage{amssymb}
\usepackage{amsfonts}
\usepackage{amsthm}
\usepackage{mathrsfs}
\usepackage{bbm}
\usepackage{color}
\usepackage{colordvi}
\usepackage{graphicx}
\usepackage{caption}
\usepackage{subcaption}
\usepackage{wrapfig}
\usepackage{amscd}
\usepackage{times}
\usepackage{framed}
\usepackage{fullpage}
\usepackage[colorlinks = true, linkcolor = magenta, citecolor = green]{hyperref}

\usepackage[boxed]{algorithm}
\usepackage{algcompatible}
\makeatletter
\newcommand{\pushright}[1]{\ifmeasuring@#1\else\omit\hfill$\displaystyle#1$\fi\ignorespaces}
\newcommand{\pushleft}[1]{\ifmeasuring@#1\else\omit$\displaystyle#1$\hfill\fi\ignorespaces}
\makeatother

\newtheorem{theorem}{Theorem}
\newtheorem{proposition}{Proposition}

\newtheorem{definition}{Definition}

\newtheorem{lemma}{Lemma}

\renewenvironment{proof}{\noindent{\textbf{Proof:}}}{$\blacksquare$\vskip\belowdisplayskip}

\newenvironment{prevproof}[2]{\noindent {\bf {Proof ({#1}~\ref{#2}):}}}{$\blacksquare$\vskip \belowdisplayskip}

\definecolor{Red}{rgb}{1,0,0}
\definecolor{Blue}{rgb}{0,0,1}
\definecolor{Olive}{rgb}{0.41,0.55,0.13}
\definecolor{Green}{rgb}{0,1,0}
\definecolor{MGreen}{rgb}{0,0.8,0}
\definecolor{DGreen}{rgb}{0,0.55,0}
\definecolor{Yellow}{rgb}{1,1,0}
\definecolor{Cyan}{rgb}{0,1,1}
\definecolor{Magenta}{rgb}{1,0,1}
\definecolor{Orange}{rgb}{1,.5,0}
\definecolor{Violet}{rgb}{.5,0,.5}
\definecolor{Purple}{rgb}{.75,0,.25}
\definecolor{Brown}{rgb}{.75,.5,.25}
\definecolor{Grey}{rgb}{.5,.5,.5}
\definecolor{Black}{rgb}{0,0,0}

\newcommand{\gcal}{G}

\newcommand{\xcal}{\mathcal{X}}

\newcommand{\eps}{\varepsilon}

\newcommand{\ind}{\mathbbm{1}}

\newcommand{\E}{\mathbb{E}}
\renewcommand{\P}{\mathbb{P}}
\newcommand{\PP}{\widetilde{\P}}
\newcommand{\EE}{\widetilde{\E}}

\def\constf{C_{f}}

\def\constj{C_{j}}

\newcommand{\remove}[1]{{}}
\newcommand{\mergeforest}{\texttt{\textbf{coalesce}}}
\newcommand{\msc}{\texttt{\textbf{MSC}}}
\newcommand{\currentvertex}{v_{\rm c}}

\newcommand{\mred}[1]{\mathcal{M}_{{\rm R}#1}}
\newcommand{\mquant}[1]{\mathcal{M}_{{\rm Q}#1}}
\newcommand{\const}[1]{c_{#1}}
\newcommand{\Econdred}{\bar{\mathbb{E}}}
\newcommand{\Pcondred}{\bar{\mathbb{P}}}

\addtocounter{page}{-1}

\newcommand{\ve}{\varepsilon}

\renewcommand{\P}{\mathbb{P}}

\begin{document}

\title{\vspace{0cm}Coalescent-based species tree estimation:\\ a stochastic Farris transform
}

\author{
Gautam Dasarathy\footnote{Department of Electrical and Computer Engineering at Rice University. \texttt{gautamd@rice.edu}}
\and Elchanan Mossel\footnote{Department of Mathematics and IDS at the Massachusetts Institute of Technology. \texttt{elmos@mit.edu}}
\and Robert Nowak\footnote{Department of Electrical and Computer Engineering at the University of Wisconsin--Madison. \texttt{rdnowak@wisc.edu}}
\and
Sebastien Roch\footnote{Department of Mathematics at the University of Wisconsin--Madison.
Supported by NSF grants DMS-1149312 (CAREER) and DMS-1614242. \texttt{roch@math.wisc.edu}}
}
\date{}
\maketitle

\begin{abstract}
The reconstruction of a species phylogeny from genomic data faces two significant hurdles: 1) the trees describing the evolution of each individual gene---i.e., the gene trees---may differ from the species phylogeny and 2) the molecular sequences corresponding to each gene often provide limited information about the gene trees themselves. 
In this paper we consider an approach to species tree reconstruction that addresses both these hurdles. Specifically, we propose an algorithm for phylogeny reconstruction under the multispecies coalescent model with a standard model of site substitution. 
The multispecies coalescent is commonly used
to model gene tree discordance due to incomplete
lineage sorting, a well-studied population-genetic effect.

In previous work, an information-theoretic trade-off was derived in this context between the number of loci, $m$, needed for an accurate reconstruction and the length of the locus sequences, $k$. It was shown that to reconstruct
an internal branch of length $f$, one needs $m$
to be of the order of $1/[f^{2} \sqrt{k}]$. That previous result was obtained under the molecular clock assumption, i.e., under the assumption that mutation rates (as well as population sizes) are constant across the species phylogeny.

Here we generalize this result beyond the restrictive molecular clock assumption, and obtain a new reconstruction algorithm that has the same data requirement (up to log factors). Our main contribution
is a novel reduction to the molecular clock case under the multispecies coalescent. As a corollary, we also obtain
a new identifiability result of independent 
interest: for any species tree with
$n \geq 3$ species, the rooted species tree can be 
identified from the distribution of its unrooted
{\em weighted} gene trees even in the absence of
a molecular clock. 
\end{abstract}

\thispagestyle{empty}

\clearpage


\section{Introduction}

Modern molecular sequencing technology has provided a wealth of data to assist biologists in the inference of evolutionary relationships between species.  Not only is it now possible to quickly sequence a single gene across a wide range of species, but in fact thousands of genes---or entire genomes---can be sequenced simultaneously. With this abundance of data comes new algorithmic and statistical challenges. One such challenge arises because phylogenomic inference entails dealing with the interplay of \emph{two} processes, as we now explain.

While the tree of life (also referred to as a species phylogeny) depicts graphically the history of speciation of living organisms, each gene within the genomes of these organisms \textit{has its own history}. That history is captured by a gene tree. In practice, by contrasting the DNA sequences of a common gene across many current species, one can reconstruct the corresponding gene tree. Indeed the accumulation of mutations along the gene tree reflects, if imperfectly, the underlying history. Much is known about the reconstruction of single-gene trees, a subject with a long history; see~\cite{SempleSteel:03,Felsenstein:04,yang2014molecular,Steel:16,Warnow:u} for an overview. The theoretical computer science literature, in particular, has contributed a deep understanding of the computational complexity and data requirements of the problem, under standard stochastic models of sequence evolution on a tree. See, e.g.,~\cite{GrahamFoulds:82,ABFPT:99,FarachKannan:99,ErStSzWa:99a,ErStSzWa:99b,Atteson:99,CrGoGo:02,SteelSzekely:02,KiZhZh:03,Mossel:03,Mossel:04a,ChorTuller:06,Roch:06,MosselRoch:06,BoChMoRo:06,Mihaescu2009,DaMoRo:11a,DaMoRo:11b,AnDaHaRo:12,GronauMS12,MiHiRa:13,DaskalakisRoch:13,RochSly:u}. 

But a gene tree is only an approximation to the species phylogeny. Indeed various evolutionary mechanisms lead to \textit{discordance} between gene trees and species phylogenies. These include the transfer of genetic material between unrelated species, hybrid speciation events and a population-genetic effect known as incomplete lineage sorting~\cite{Maddison:97}. The wide availability of genomic datasets has brought to the fore the major impact these discordances have on phylogenomic inference~\cite{DeBrPh:05,DegnanRosenberg:09}. 
As a result, in addition to the stochastic process governing the evolution of DNA sequences on a fixed gene tree, one is led to model the structure of the gene tree \textit{itself}, in relation to the species phylogeny, through a separate stochastic process. The inference of these complex, two-level evolutionary models is an active area of research. See the recent monographs~\cite{huson2010phylogenetic,Steel:16,Warnow:u} for
an introduction.

In this paper, we focus on incomplete lineage sorting (from hereon ILS) and consider the reconstruction of a species phylogeny
from multiple genes (or loci) under a standard population-genetic
model known as the multispecies coalescent~\cite{RannalaYang:03}. 
The problem is of
great practical interest
in computational evolutionary biology
and is currently the subject of intense study;
see e.g.~\cite{LYK+:09,DegnanRosenberg:09,ALP:+:12,Nakhleh:13} for a survey.
There is in particular a growing body of theoretical results in this area~\cite{DegnanRosenberg:06,DeDGiBr+:09,DeGiorgioDegnan:10,MosselRoch:10a,LiYuPe:10,AlDeRh:11,AlDeRh:11b,Roch13,DaNoRo:14,RochSteel:14,DeGDe:14,CHIFMAN2015,RochWarnow:15,MosselRoch:15,AlDeRh:17,ShekharRochMirarab:u}, although
much remains to be understood. This inference problem
is also closely related to another active
area of research, the reconstruction
of demographic history in population genetics.
See e.g.~\cite{MyFePa:08,BhaskarSong:14,KMR+:15}
for some recent theoretical results.

A significant fraction of prior rigorous work on species phylogeny estimation in the presence of ILS has been aimed at the case where ``true'' gene trees are assumed to be available. However, in reality, one needs to estimate gene trees from DNA sequences, leading to reconstruction errors, and indeed there has been a recent thrust towards understanding the effect of this important source of error in phylogenomic inference, both from empirical~\cite{KubatkoDegnan:07,MiBaWa:16} and theoretical~\cite{MosselRoch:10a,DeGDe:14,RochSteel:14,RochWarnow:15,ShekharRochMirarab:u} standpoints. Another option, which we adopt here, is
to bypass the reconstruction of gene trees altogether and infer the species history directly from sequence data~\cite{dasarathy2015data,MosselRoch:15,CHIFMAN2015}. 

In previous work on this latter approach~\cite{MosselRoch:15}, an optimal information-theoretic 
trade-off was derived between the
number of genes $m$ needed to accurately reconstruct a species phylogeny
and the length of the genes $k$ (which is linked to the quality of the phylogenetic signal that can be extracted from each
separate gene). Specifically, it was shown that $m$ needs to scale like $1/[f^{2} \sqrt{k}]$, where $f$ is the length of the shortest branch in the tree (which controls the extent of the ILS). This result was obtained under a restrictive \emph{molecular clock} assumption, where the leaves are equidistant from the root; in essence, it was assumed that the mutation rates and population sizes do not vary across the species phylogeny, which is rarely the case in practice.

In the current work, we design and analyze a new reconstruction algorithm that achieves the same optimal data requirement (up to log factors) beyond the molecular clock assumption. Our key contribution is of independent interest: we show how to transform sequence data to appear as though it was generated under the multispecies coalescent with a molecular clock. We achieve this through a novel reduction which we call a \emph{stochastic Farris transform}. Our construction relies on
a new identifiability result: for any species phylogeny with
$n \geq 3$ species, the rooted species tree can be 
identified from the distribution of the unrooted
{\em weighted} gene trees even in the absence of
a molecular clock. 

We state our main results formally in Section~\ref{sec:back} and describe our new reduction in Section~\ref{sec:reduction-step}. The proofs are in Sections~\ref{sec:ident-proofs},~\ref{sec:main-proof},~\ref{sec:app-redux} and~\ref{sec:app-quantile}.

\section{Background and main results}
\label{sec:back}

In this section, we state formally our main results and provide a high-level view of the proof.

\subsection{Basic definitions}

We begin with a brief description of our modeling assumptions. More details on the models, which are standard in the phylogenetic literature (see e.g.~\cite{Steel:16}), are provided in Section~\ref{section:models}.

\paragraph{Species phylogeny v. gene trees}
A species phylogeny is a graphical depiction of the evolutionary history of a set of species. The leaves of the tree correspond to extant species while internal vertices indicate a speciation event. Each edge (or branch) corresponds to an ancestral population and will be described here by two numbers: one that indicates the amount of time that the corresponding population lived, and a second one that specifies the rate of genetic mutation in that population. Formally, we define the species phylogeny (or tree) as follows. 
\begin{definition}[Species phylogeny]
\label{def:species-tree+metric}
	A species phylogeny $S = (V_s, E_s; r, \vec{\tau},\vec{\mu})$ is a directed tree rooted at $r\in V_s$ with vertex set $V_s$, edge set $E_s$, and $n$ labelled leaves $L = \left\{ 1,2,\ldots,n \right\}$ such that (a) the degree of all internal vertices is $3$ except for the root $r$ which has degree $2$, and (b) each edge $e\in E_s$ is associated with a length $\tau_e\in (0,\infty)$ and a mutation rate $\mu_e \in (0,\infty)$. 
\end{definition}
\noindent To be more precise, as is standard in coalescent theory (see, e.g., \cite{Steel:16}), the length $\tau_e$ of a branch $e\in E_s$ is expressed in coalescent time units, which is the duration of the branch $t_e$ divided by its population size $N_e$. That is, $\tau_e = t_e/N_e$. 
We pictorially represent species phylogenies as thick shaded trees; see Fig.~\ref{fig:weighted-vs-unweighted-species-tree} for an example with $n=3$ leaves. 
\begin{figure*}[h!!]
    \centering
        \includegraphics[scale=2.8]{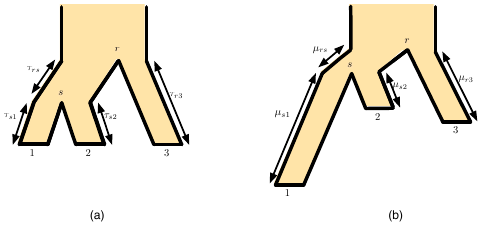}
\caption{(a) A species phylogeny with $n=3$ leaves $\{1,2,3\}$ and two internal vertices $r$ and $s$. The branch lengths are denoted by $\{\tau_{s1}, \tau_{s2}, \tau_{rs}, \tau_{r3}\}$. Depicted above is the special case where all population sizes are the same, in which case the species phylogeny is ``ultrametric,'' that is, all leaves are the same distance from the root under $\vec{\tau}$. (b) The same phylogeny with branches stretched by the corresponding mutation rates. The mutation rate-weighted branch lengths are denoted by $\{\mu_{s1}, \mu_{s2}, \mu_{rs}, \mu_{r3}\}$, with respect to which ultrametricity is in general lost---our focus here.}
\label{fig:weighted-vs-unweighted-species-tree}
\end{figure*}

While a species phylogeny describes the history of speciation, each gene has its own history which is captured by a gene tree.  
\begin{definition}[Gene trees]
\label{def:gene-tree}
A gene tree $G^{(i)} = (V^{(i)}, E^{(i)}; r, \vec{\delta}^{(i)})$	corresponding to gene $i$ is a directed tree rooted at $r$ with vertex set $V^{(i)}$ and edge set $E^{(i)}$, and the same labeled leaf set $L = \left\{ 1,2,\ldots, n \right\}$ as $S$ such that (a) the degree of each internal vertex is $3$, except the root $r$ whose degree is $2$, and (b) each branch $e\in E^{(i)}$ is associated with a branch length $\delta_e^{(i)}\in (0,\infty)$.
\end{definition}
\noindent In essence, these gene trees ``evolve'' on the species phylogeny. More specifically,
following \cite{rannala2003bayes}, we assume that a multispecies coalescent (MSC) process produces $m$ independent random gene trees $G^{(1)}, G^{(2)}, \ldots, G^{(m)}$. This process is parametrized by the species phylogeny $S$. 
In words, proceeding backwards in time, in each population, every pair of lineages entering from descendant populations merge at a unit exponential rate. We describe this process formally in Algorithm~\ref{alg:msc-process} in Section~\ref{app:msc-details}. For the present discussion, it suffices to think of the MSC as a random process generating ``noisy versions'' of the species phylogeny. We highlight one key feature of the gene trees: their topology may be \emph{distinct from that of the species phylogeny}.
This discordance, which in this context is referred to as incomplete lineage sorting (see e.g.~\cite{DegnanRosenberg:09}), is a major challenge for species tree estimation from multiple genes. See Figure~\ref{fig:msc-example-alg} for an illustration in the case $n=3$.
\begin{figure*}[h!]
	\begin{center}
		\includegraphics[scale = 1.75]{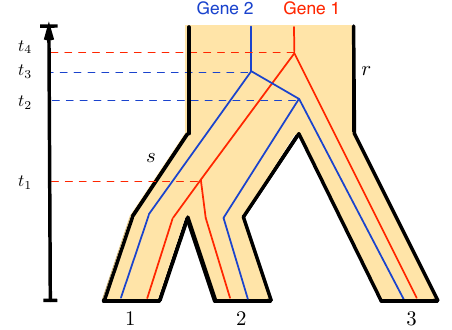}
	\end{center}
	\caption{Two sample draws from the MSC on a species phylogeny with $n=3$ leaves.  The (rooted) topology of Gene 1 (red gene) agrees with the topology of the underlying species phylogeny (i.e., species $1$ and $2$ are closest), while the topology of Gene 2 (blue gene) does not (here species $2$ and $3$ are closest instead). 
	}
	\label{fig:msc-example-alg}
\end{figure*}

\paragraph{Sequence data and inference problem}
The gene trees are not observed directly. Rather, they are typically inferred from sequence data ``evolving'' on the gene trees.
We model this sample generation process according to the standard Jukes-Cantor (JC) model (see, e.g., \cite{Steel:16}). That is, given a gene tree $G^{(i)} = (V^{(i)},E^{(i)}; r, \vec{\delta}^{(i)})$, we associate to each $e\in E^{(i)}$, a probability 
\begin{equation*}
p_{e}^{(i)} = \frac{3}{4}\left( 1 - e^{-\frac{4}{3}\delta_{e}^{(i)}} \right),	
\end{equation*}
where $\delta_{e}^{(i)}$ is the mutation rate-weighted edge length.
In words, the corresponding gene $i$ is a sequence of length $k$ in $\left\{{\tt A, T, G, C}\right\}^k$. Each position in the sequence evolves independently, starting from a uniform state in $\left\{{\tt A, T, G, C}\right\}$ at the root. Moving away from the root,
a substitution occurs on edge $e$ with probability $p_{e}^{(i)}$,
in which case the state changes to state chosen uniformly among the remaining states. After repeating this process for all positions, one obtains a sequence of length $k$ for each leaf of $G^{(i)}$, for each $i \in [m]$---that is our input.
A full algorithmic description of the Jukes-Cantor process is provided as Algorithm~\ref{alg:jukes-cantor} in Section~\ref{section:models}.

For gene $i$, we will let $\left\{ \xi^{ij}_x \,:\, j\in [k], x\in L \right\}$ denote the data generated at the leaves $L$ of the tree $G^{(i)}$ per the Jukes-Cantor process, the superscript $j$ runs across the positions of the gene sequence. To simplify the notation, we denote $\xi^{ij} = (\xi^{ij}_x)_{x \in L}$. 
The species phylogeny estimation problem can then be stated as follows:~
\begin{quote}
	{\em The $n\times m\times k$ data array $\left\{ \xi^{ij} \right\}_{i\in [m], j\in [k]}$ is generated according to the Jukes-Cantor process on the $m$ gene trees, each of which in turn is generated according to the multi-species coalescent on $S$. The goal is to recover the topology of the species phylogeny $S$ from  $\left\{ \xi^{ij} \right\}_{i\in [m], j\in [k]}$.}
\end{quote}
We abbreviate this \emph{two-step} data generation process by saying that $\left\{ \xi^{ij} \right\}_{i\in [m], j\in [k]}$ is generated according to the MSC-JC$(m,k)$ process on $S$. 

\subsection{Main result}
\label{sec:main-result}

\begin{algorithm}[h!]
	\caption{Quantile-based triplet test}
	\label{alg:questAdditiveQuantile}
	\begin{algorithmic}[1]
		\REQUIRE Sequence output by Algorithm~\ref{alg:questAdditiveReduction} $\{\xi_{x,N}^{ij}\,:\, x\in \mathcal{X} = \left\{ 1,2,3 \right\}, i\in \mquant{} ,j\in [k]\}$. A partition of the set of genes $[m] = \mred{} \sqcup \mquant{}$, where $\mred{} = \mred1\sqcup\mred2$ and $\mquant{} = \mquant1\sqcup \mquant2$ satisfy the conditions of Proposition~\ref{thm:main-full}.
		\STATE For each $x,y \in \mathcal{X}$ and $i\in \mquant{}$, let $\widehat{q}_{xy}^i = \sum_{j=1}^k \ind\{\xi_{x,N}^{ij}\neq \xi_{y,N}^{ij}\}$.
		\STATE Set $\alpha \triangleq \max\left\{m^{-1} \log m,k^{-0.5} \sqrt{\log k}\right\}$, and partition $\mquant{} = \mquant{1} \sqcup \mquant{2}$ such that $\left| \mquant{1} \right|, \left| \mquant{2} \right|$ satisfy the conditions in Proposition~\ref{thm:main-full}.
		\vspace{4mm}		
		\STATEx \underline{\bf Ultrametric Quantile Test on $\left\{ \widehat{q}_{xy}^i : {x,y \in \mathcal{X}, i\in \mquant{1}}\right\}$}
		\STATE For each pair of leaves $x,y \in \mathcal{X}$, compute $\widehat{q}^{(\const{3} \alpha)}_{xy}$, the $\const{3} \alpha$-th quantile with respect to the data $\left\{ \widehat{q}_{xy}^i: i\in \mquant{1} \right\}$. 
		The constant $\const{3}$ is as in Proposition~\ref{prop:quantileBehavior}. Define 
		\begin{equation*}
		\widehat{q}_{\ast} \triangleq \max_{x,y \in \left\{ 1,2, 3 \right\}}\widehat{q}^{(\const{3} \alpha)}_{xy}.
		\end{equation*}
		\STATE Next, for $x,y \in \mathcal{X}$, define a similarity measure
		\begin{equation*}
		\widehat{s}_{xy} \triangleq \frac{1}{\left| \mquant{2} \right|} \left|\left\{ i\in \mquant{2}: \widehat{q}^{i}_{xy}\leq \widehat{q}_{\ast} \right\}\right|.
		\end{equation*}
		\STATEx {\hspace{-7mm}\bf Return} Declare that the topology is $xy|z$ if $\widehat{s}_{xy} > \max\left\{ \widehat{s}_{xz}, \widehat{s}_{yz} \right\}.$
	\end{algorithmic}
\end{algorithm}

\begin{algorithm}[h!]
	\caption{Reduction step}
	\label{alg:questAdditiveReduction}
	\begin{algorithmic}[1]
		\REQUIRE Sequences $\{\xi_x^{ij}\,:\, x\in \mathcal{X} = \left\{ 1,2,3 \right\}, i\in [m],j\in [k]\}$. A partition of the set of genes $[m] = \mred{} \sqcup \mquant{}$, where $\mred{} = \mred1\sqcup\mred2$ and $\mquant{} = \mquant1\sqcup \mquant2$ satisfy the conditions of Proposition~\ref{thm:main-full}. 
		\STATE For each $x,y\in \mathcal{X}$ and $i\in \mred{}$, define $\widehat{p}_{xy}^i = \frac{1}{k} \sum_{j = 1}^k \ind\{\xi^{ij}_{x} \neq \xi^{ij}_{y}\}$, $\widehat{p}^{i\downarrow}_{xy} = \frac{2}{k} \sum_{j = 1}^{k/2} \ind\{\xi^{ij}_{x} \neq \xi^{ij}_{y}\} $, and $\widehat{p}^{i\uparrow}_{xy} = \frac{2}{k} \sum_{j = k/2 + 1}^{k} \ind\{\xi^{ij}_{x} \neq \xi^{ij}_{y}\}$. 
		\STATE Let $\{x_i, y_i\}$, $i=1,2,3$, be the three distinct (unordered) pairs of distinct leaves in $\mathcal{X}$.
		\FOR {i = 1,2}
		
		\vspace{2mm}
		\STATEx \quad\underline{\bf Fixing gene tree topologies}\vspace{1mm}
		
		\STATE Let $x = x_i$ and $y = y_i$.
		\STATE Let $z$ be the unique element in $\xcal - \{x,y\}$.
		\STATE Compute the empirical quantiles $\widehat{p}_{xy}^{(1/3)}$, $\widehat{p}_{xz}^{(2/3)}$, $\widehat{p}_{xz}^{(5/6)}$, $\widehat{p}_{yz}^{(2/3)}$, and $\widehat{p}_{yz}^{(5/6)}$ from the loci
		in $\mred{1}$. 
		For instance, to compute $\widehat{p}_{xy}^{(1/3)}$, sort the set $\left\{ \widehat{p}^{i}_{xy}: i\in \mred{1} \right\}$ in ascending order and pick the $\left\lfloor \frac{\left| \mred{1} \right|}{3} \right\rfloor$-th element, breaking ties arbitrarily.  
		\STATE Set $I := \left\{ i\in \mred{2}: \widehat{p}_{xy}^{i\downarrow} \leq \widehat{p}_{xy}^{(1/3)}, \widehat{p}_{xz}^{(2/3)}\leq \widehat{p}_{xz}^{i\downarrow}\leq \widehat{p}_{xz}^{(5/6)}, \widehat{p}_{yz}^{(2/3)}\leq \widehat{p}_{yz}^{i\downarrow} \leq \widehat{p}_{yz}^{(5/6)}\right\}$. 
		
		\vspace{4mm}
		\STATEx \quad\underline{\bf Estimation of differences $\Delta_{xy}$}\vspace{1mm}
		
		\STATE Set $\widehat{p}^I_{xz} := \frac{1}{\left| I \right|}\sum_{i\in I}\widehat{p}_{xz}^{i\uparrow}$, and similarly for $\widehat{p}_{yz}^I$. 
		\STATE Set $\widehat{\Delta}_{xy} := - \widehat{\Delta}_{yx} := -\frac{3}{4}\log\left( \frac{1 - \frac{4}{3}\widehat{p}_{yz}^I}{1 - \frac{4}{3}\widehat{p}_{xz}^I} \right)$
		\ENDFOR
		\STATE Let $z_3$ be the unique element in $\xcal - \{x_3,y_3\}$.
		\STATE Set $\widehat{\Delta}_{x_3 y_3} := \widehat{\Delta}_{x_3 z_3} - \widehat{\Delta}_{y_3 z_3}$.
		\vspace{4mm}
		\STATEx \underline{\bf Stochastic Farris transform}\vspace{1mm}
		\STATE Find a permutation $\{x,y,z\}$ of $\mathcal{X}$ such that $\min\{\widehat{\Delta}_{zx}, \widehat{\Delta}_{zy}\} \geq  0$. 
		\STATE 
		For each gene $i\in \mquant{}$ and $j \in [k]$, set $\xi_{z,N}^{ij} = \xi^{ij}_z$. Also set $\xi_{x,N}^{ij} = \xi^{ij}_x$ with probability $ 1- p (\widehat{\Delta}_{zx})$ and otherwise choose $\xi_{x,N}^{ij}$ uniformly from $\{{\tt A, T, G, C}\}\setminus \xi^{ij}_{x}$. Do the same to $\xi^{ij}_{y}$ (with $\widehat{\Delta}_{yz}$ instead of $\widehat{\Delta}_{xz}$) to obtain $\xi^{ij}_{y,N}$.
		\vspace{4mm}

		\STATEx {\hspace{-7mm}\bf Return} ``noisy'' sequence data $\left\{\xi^{ij}_{x,N}: i\in \mquant{}, j\in [k], x\in \mathcal{X}  \right\}$
	\end{algorithmic}
\end{algorithm}

We now state our main result for the species phylogeny estimation problem. 
For any 3 leaves $x, y, z \in L$, the species phylogeny $S$ restricted to these three leaves has one of three possible rooted topologies: $xy|z$, $xz|y$, or $yz|x$. For instance, $12|3$ is depicted in Figure~\ref{fig:weighted-vs-unweighted-species-tree} (a) and indicates that $1$ and $2$ are closest. It is a classical phylogenetic result that if one is able to correctly reconstruct the topology of all triples of leaves in $L$, then the topology of the full species phylogeny can be correctly reconstructed as well (see e.g., \cite{Steel:16}).
Therefore, to simplify the presentation, in what follows our algorithms and theoretical guarantees are stated for a fixed triple $\mathcal{X} = \{1, 2, 3\}$ (without loss of generality) among the set of leaves $L$. 

Our main contribution is a \emph{novel polynomial-time reconstruction algorithm} for the species phylogeny estimation problem, along with a \emph{rigorous data requirement which is optimal} (up to log factors) by the work of~\cite{MosselRoch:15}. Moreover, unlike~\cite{MosselRoch:15}, our results hold when \emph{mutation rates and populations sizes are allowed to vary} across the species phylogeny. Our reconstruction algorithm comprises two steps, which are detailed as Algorithm~\ref{alg:questAdditiveQuantile} and Algorithm~\ref{alg:questAdditiveReduction}. Our data requirement applies to an unknown species phylogeny in the following class. We assume that: mutation rates are in the interval $(\mu_L, \mu_U)$; leaf-edge lengths are in $(f', g')$; and internal-edge lengths are in $(f, g)$. We suppress the dependence on $\mu_L, \mu_U, f', g', g$, which we think of as constants, and focus here on the role of $f$. The latter indeed plays a critical role in both the random processes described above. Short internal branches are known to be hard to reconstruct from sequence data even when dealing with a single gene tree~\cite{SteelSzekely:02} and a smaller $f$ also leads to more discordance between gene trees~\cite{rannala2003bayes}. We also suppress the dependence on the number of leaves $n = |L|$, which we also consider here to be a constant (see the concluding remarks in Section~\ref{sec:conclusion} for more on this).

We state here a simplified version of our results (the more general statement appearing as Proposition~\ref{thm:main-full} in Section~\ref{sec:main-proof}). Specifically, we answer the following question: as $f \to 0$, how many genes $m$ of length $k$ are needed for a correct reconstruction with high probability? For technical reasons, our results apply only when $k$ grows at least polynomially with $f$ (with an arbitrarily small exponent).
Throughout, we use the notation $\gtrsim$ (similarly, $\lesssim$) to indicate that constants and $\text{poly}(\log f^{-1})$ factors are suppressed in a lower bound. Recall that $x \lor y = \max\{x, y\}$.
\begin{theorem}[Data requirement]\label{thm:main}
Suppose that we have sequence data $\left\{ \xi^{ij} \right\}_{i\in [m], j\in [k]}$ generated according to the MSC-JC$(m,k)$ process on a species phylogeny $S = (V_s, E_s; r, \vec{\tau},\vec{\mu})$ . The mutation rates, leaf-edge lengths and internal-edge lengths are respectively in $(\mu_L, \mu_U)$, $(f', g')$ and $(f, g)$.  We assume further that there is $C > 0$ such that $k \gtrsim  f^{-C}$. Then Algorithm~\ref{alg:questAdditiveQuantile} correctly identifies the topology of $S$ restricted to $\mathcal{X} = \{1,2,3\}$ with probability at least $95\%$ provided that
\begin{equation}
\label{eq:main-condition}
m \gtrsim 
\frac{1}{f} \lor \frac{1}{\sqrt{k} f^2}.
\end{equation}
\end{theorem}
\noindent Two regimes are implicit in Theorem~\ref{thm:main}:
\begin{itemize}
	\item \textbf{``Long'' sequences:} When $k \gtrsim f^{-2}$, 
	we require $m \gtrsim f^{-1}$. As first observed by~\cite{MosselRoch:10a}, this condition is always required for high-probability reconstruction under this setting. 
	
	\item \textbf{``Short'' sequences:} When $k \lesssim f^{-2}$, 
	we require the stronger condition that $m \gtrsim k^{-1/2} f^{-2}$. This is known to be optimal (up to the log factor) by the information-theoretic lower bound in~\cite{MosselRoch:15}.
	As mentioned above, the matching algorithmic upper bound of~\cite{MosselRoch:15} only applies when all mutation rates and population sizes are identical. Our main contribution here is to relax this assumption.
\end{itemize}
On the other hand, our results do not apply to the regime of ``very short'' sequences of constant length. In that regime, the reconstruction algorithm of~\cite{dasarathy2015data}, which applies under the same setting we are considering here, achieves the optimal bound of $m \gtrsim f^{-2}$.

\subsection{Proof idea and further results}

We give a brief overview of the proof. The full details are given in Section~\ref{sec:reduction-step} as well as Sections~\ref{sec:algo-high-level-proof},~\ref{sec:app-redux} and~\ref{sec:app-quantile}.
Again, fix a triple of leaves $\mathcal{X} = \{1, 2, 3\}$.

\paragraph{Tree metrics}
Phylogenies are naturally equipped with a notion of distance between leaves, and in general any pair of vertices, which is known as a tree metric (see e.g.~\cite{Steel:16} for more details). Our species phylogeny reconstruction method rests on such tree metrics. 
\begin{definition}[Weighted species metric]
	A species phylogeny $S = (V_s, E_s; r, \vec{\tau}, \vec{\mu})$ induces the following metric on the leaf set $L$. For any pair of leaves $a,b\in L$, we let 
	\begin{align*}
	\mu_{ab} &= \sum_{e\in \pi(a,b;S)}\tau_e\, \mu_e,
	\end{align*}
	where $\pi(a,b;S)$ is the unique path connecting $a$ and $b$ in $S$ interpreted as a set of edges. We will refer to $\left\{ \mu_{ab} \right\}_{a,b\in L}$ as the {\bf weighted species metric} induced by $S$.  
\end{definition}
\noindent  The above definition is valid for any pair of vertices in $V_s$. That is, the metric $\mu$ can be extended to the entire set $V_s$. 
In the species phylogeny estimation problem, the sequence data only carries information about the rate-weighted distances $\left\{ \mu_{ab} \right\}_{a,b\in L}$. The algorithm in \cite{MosselRoch:15} is guaranteed to recover the topology of $S$ only in the case that $\left\{ \mu_{ab} \right\}_{a,b\in L}$ is an ultrametric on the leaf set $L$, in which case we refer to $S$ as an ultrametric species phylogeny. The metric $\left\{ \mu_{ab} \right\}_{a,b\in L}$ is 
ultrametric when $\mu_{ra} = \mu_{rb}$ for all $a, b \in L$, that is, when the distance from the root to every leaf is the same.

Recall from Definition~\ref{def:gene-tree} that each each random gene tree has an associated set of branch lengths. From the description of the multispecies coalescent (see Section~\ref{section:models}), it follows that a single branch of a gene tree may span across multiple branches of the species phylogeny; this can also be seen in Fig.~\ref{fig:msc-example-alg}. Let $t_{\tilde{e}}$ denote the (random) length of the branch $\tilde{e}\in E^{(i)}$.
For any species phylogeny branch $e\in E_s$, let $t_{\tilde{e}\cap e}$ denote the length of the branch $\tilde{e}$ that overlaps with $e$. Then, 
$\delta_{\tilde{e}}$ and $t_{\tilde{e}}$ satisfy the following relationship
\begin{equation*}
\delta_{\tilde{e}} = \sum_{e\in E_s} \mu_e t_{\tilde{e}\cap e}.
\end{equation*}
This set of weights again defines a different metric on the leaves $L$ of the species tree.  
\begin{definition}[Gene metric]
	\label{def:gene-metric}
	A gene tree $G^{(i)} = (V^{(i)}, E^{(i)}; r, \vec{\delta}^{(i)})$ induces the following metric on the leaf set $L$. For any pair of leaves $a,b\in L$, we (overload the notation $\delta$) and let 
	\begin{align*}
	\delta_{ab}^{(i)} &= \sum_{e\in \pi(a,b;G^{(i)})}\delta_e^{(i)}
	\end{align*}
	where, again, $\pi(a,b;G^{(i)})$ is the unique path connecting $a$ and $b$ in $G^{(i)}$ interpreted as a set of edges. We will refer to $\left\{ \delta_{ab}^{(i)} \right\}_{a,b\in L}$ as the {\bf gene metric} induced by $G^{(i)}$.  
\end{definition}
\noindent Note that, when the species phylogeny $S$ is ultrametric, so are the gene trees.

\paragraph{Ultrametric reduction}
At a high level, our reconstruction algorithm relies on a quantile triplet
test developed in~\cite{MosselRoch:15}. Roughly speaking this test, which is detailed in Algorithm~\ref{alg:questAdditiveQuantile}, compares a well-chosen quantile of the sequence-based estimates of gene metrics $\left\{ \delta_{ab}^{(i)} \right\}_{a,b\in \mathcal{X}}$ in order to determine which pair of leaves is closest. The algorithm of~\cite{MosselRoch:15}, however, only works when all mutation rates and population sizes are equal. In that case, the species phylogeny and gene trees are ultrametric, as defined above. That property leads to symmetries that play a crucial role in the algorithm. Our first main contribution here is a reduction to the this ultrametric case.

That is,  in order to apply the quantile triplet test, we first transform the sequence data to appear as though it was was generated by an ultrametric species phylogeny. This ultrametric reduction, inspired by a classical technique known as the Farris transform (see e.g.~\cite{semple2003phylogenetics}), may be of independent interest as it could be used to generalize other reconstruction algorithms. 
Formally, we prove the following theorem.
Again, we state a simplified version of our result which gives a lower bound on the number of genes $m$ of length $k$ needed to achieve a desired accuracy (the more general statement appearing as Proposition~\ref{thm:reduction-full} in Section~\ref{sec:app-redux}). More specifically, Algorithm~\ref{alg:questAdditiveReduction} takes as input two sets of genes, $\mred{}$ and $\mquant{}$. The set $\mred{}$ is used to estimate parameters needed for the reduction. The reduction is subsequently performed on $\mquant{}$. We let $m' = |\mred{}|$. Here we give a lower bound on $m'$ (while, for the purposes of this theorem, $|\mquant{}|$ can be arbitrarily large). For $\theta>0$, we say that two metrics $\mu'$ and $\mu''$ over $\mathcal{X}$ are $\theta$-close if $\left| \mu'_{xy} - \mu''_{xy} \right|\leq \theta$, for all $x, y\in \mathcal{X}$.
\begin{theorem}[Ultrametric reduction]
	\label{thm:reduction}
	Suppose that we have sequence data $\left\{ \xi^{ij} \right\}_{i\in [m], j\in [k]}$ generated according to the MSC-JC$(m,k)$ process on a three-species phylogeny $S = (V_s, E_s; r, \vec{\tau},\vec{\mu})$. The mutation rates, leaf-edge lengths and internal-edge lengths are respectively in $(\mu_L, \mu_U)$, $(f', g')$ and $(f, g)$.  We assume further that there is $C > 0$ such that $k \gtrsim f^{-C}$. Then, with probability at least $95\%$,
	the output
	of Algorithm~\ref{alg:questAdditiveReduction}
	is distributed according to the MSC-JC process on a species tree $S'$ that is $\phi$-close to an ultrametric species phylogeny with rooted topology identical to that of $S$ restricted to $\mathcal{X} = \{1,2,3\}$, provided that
	\begin{equation}
	\label{eq:reduction-condition}
	m' \gtrsim 
	1 \lor \frac{1}{k f^2},
	\end{equation}
	where $\phi = \Theta(f/\log f^{-1})$.
\end{theorem}
\noindent 
The log factor in $\phi$ is needed in our analysis of the quantile test below.
The key to the proof of Theorem~\ref{thm:reduction} is the establishment of a new identifiability result of independent interest. 
\begin{theorem}[Identifiability of rooted species tree from unrooted weighted gene trees]
	\label{thm:ident}
	Let $S = (V_s,E_s; r, \vec{\tau}, \vec{\mu})$ be a species tree
	with $n \geq 3$ leaves and root $r$
	and let $G=(V,E; r, \vec{\delta})$
	be a sampled gene tree from the MSC with
	branch lengths $\delta_e$, $e \in E$. 
	Then the rooted topology of the species
	tree $S$ is identifiable from the
	distribution of the unrooted weighted
	gene tree $G$.
\end{theorem}
\noindent 
The case $n \geq 5$ is not new. Indeed, it follows from~\cite[Theorem 9]{AlDeRh:11}, where it is shown that in fact the distribution of the unrooted gene tree {\em topologies} (without any branch length information) suffices to identify the rooted species phylogeny when the number of leaves exceeds 4. On the other hand, it was also shown in~\cite[Proposition 3]{AlDeRh:11} that, when $n=4$, the gene tree topologies are not enough to locate the root of the species phylogeny (and the case $n=3$ is trivial).
Here we show that, already with three species (and therefore also when $n > 3$), the extra information in the gene tree branch lengths allows to recover the root.
We give a constructive proof of Theorem~\ref{thm:ident}, which we then adapt to obtain Algorithm~\ref{alg:questAdditiveReduction}.
More details on this key step are given in Section~\ref{sec:reduction-step}.

\paragraph{Robustness of quantile test} 
Algorithm~\ref{alg:questAdditiveReduction} produces a new sequence dataset $\left\{ \xi_{x,N}^{ij}: x\in \mathcal{X} \right\}$ that appears close to being distributed according to an ultrametric species phylogeny. The next step is to perform a triplet test of~\cite{MosselRoch:15}, detailed in Algorithm~\ref{alg:questAdditiveQuantile}. 
Roughly speaking, this test is based on comparing an appropriately
chosen quantile of the gene metrics. In fact, because we do not have direct access to the latter, we use a sequence-based surrogate, the empirical $p$-distances 
\begin{equation*}
\widehat{q}^i_{xy} = \frac{1}{k}\sum_{j=1}^k\ind\left\{ \xi_{x,N}^{ij}\neq \xi_{y,N}^{ij} \right\},
\end{equation*}
for each gene $i \in \mquant{}$ in the output of the reduction,
whose expectation is a monotone transformation of the corresponding gene metrics.
The idea of Algorithm~\ref{alg:questAdditiveQuantile} is to use the above $p$-distances to define a ``similarity measure'' $\widehat{s}_{xy}$ between each pair of leaves $x,y\in \mathcal{X}$ to reveal the underlying species tree topology on $\mathcal{X}$.  
It works as follows. 
The set of genes $\mquant{}$ is divided into two disjoint subsets
$\mquant{1}, \mquant{2}$. The set $\mquant{1}$ is used to compute the $c_3 \alpha$-quantile $\widehat{q}^{(\const{3} \alpha)}_{xy}$ of $\{\widehat{q}^i_{xy}\,:\,i\in \mquant{1}\}$,
where $c_3 > 0$ is a constant determined in the proofs and
$$\alpha = \max\left\{ \frac{\log m}{m}, \sqrt{\frac{\log k}{k}} \right\}.$$
Let $\widehat{q}_\ast$ denote the maximum among $\left\{ \widehat{q}^{(\const{3} \alpha)}_{xy}: x,y\in \mathcal{X} \right\}$. We then use the genes in $\mquant{2}$ to define the similarity measure
\begin{equation*}
\widehat{s}_{xy} = \frac{1}{\left| \mquant{2} \right|}\left| \left\{ i\in \mquant{2}: \widehat{q}^i_{xy} \leq \widehat{q}_\ast \right\} \right|.
\end{equation*}
Whichever pair $x, y \in \mathcal{X}$ produces the largest value of $\widehat{s}_{xy}$ is declared the closest, i.e., the output is $xy|z$ where $z$ is the remaining leaf in $\mathcal{X}$. 

Why does it work? Intuitively, the closest pair of species $x,y$ will tend to produce a larger number of genes with few differences between their sequences at $x$ and $y$, as measured by the $p$-distance. In fact it was shown in~\cite{MosselRoch:10a} that, under the MSC-JC process on an ultrametric phylogeny when sequences are long enough (namely $k \gtrsim f^{-2}$), choosing the pair of species achieving the smallest $p$-distance across genes succeeds with high probability under optimal data requirements. When $k$ is short on the other hand (namely $k \lesssim f^{-2}$), the randomness from the JC process produces outliers that confound this approach. To make the test more robust, it is natural to turn to quantiles, i.e., to remove a small, fixed fraction of outliers. On a fixed gene tree, the standard deviation of the $p$-distance is of order $1/\sqrt{k}$. It was shown in~\cite{MosselRoch:15} that, as a result, $1/\sqrt{k}$ is in a sense the smallest quantile that can be meaningfully controlled and that it leads to a successful test under optimal data requirements. Our choice of quantile $\alpha$ is meant to cover both regimes above simultaneously. See~\cite{MosselRoch:15}, as well as~\cite{MosselRoch:10a,dasarathy2015data}, for more details.

As stated in Theorem~\ref{thm:reduction}, the output to the ultrametric reduction is almost---but not perfectly---ultrametric.
In our second main contribution, to account for this extra error, we perform a delicate robustness analysis of the quantile-based triplet test. This step is detailed in Section~\ref{sec:quantile-test}. At a high level, the proof follows~\cite{MosselRoch:15}. After 1) controlling the deviation of the quantiles, we establish that 2) the test works in expectation and then 3) finish off with concentration inequalities. All these steps must be updated to account for the error introduced in the reduction step. Step 2) is particularly involved and requires the delicate analysis of the CDF of a mixture of binomials.


\section{Key ideas in the ultrametric reduction}
\label{sec:reduction-step}

The goal of the ultrametric reduction step, Algorithm~\ref{alg:questAdditiveReduction}, is to transform the sequence data to appear statistically as though it is the output of an MSC-JC process on an ultrametric species phylogeny with the same topology as $S$ restricted to $\mathcal{X}$.

\subsection{Preliminary step: a new identifiability result}
\label{sec:identifiability-result}

Before diving into the description of Algorithm~\ref{alg:questAdditiveReduction},
we provide some insights into the algebra of our reduction by first deriving a new identifiability result, Theorem~\ref{thm:ident}. That is,
we show that, under the multispecies coalescent, the rooted topology of the species phylogeny
can be recovered from the distribution
of the unrooted weighted gene trees.

Our reduction is inspired by the Farris transform (also related to the Gromov product; see e.g.~\cite{semple2003phylogenetics}),
a classical technique to transform a general metric into an ultrametric. 
In a typical application of the Farris transform, one ``roots'' the species phylogeny $S$ at an ``outgroup'' $o$ (i.e., a species that is ``far away'' from the leaves of $S$) and then uses the quantities $\mu_{ox}, x\in L$ to implicitly stretch the leaf edges appropriately, so as to make all inter-species distances to $o$ equal, without changing the underlying topology. More specifically, let $S$ be a species phylogeny. Suppose $\mathcal{X} = \{1,2,3\}$ and let $o \in L-\mathcal{X}$ be any leaf of $S$ outside $\mathcal{X}$. Assume that
$
\mu_{o1} \geq \max\{\mu_{o2}, \mu_{o3}\}
$
(the other cases being similar) and
define the Farris transform
\begin{equation}
\label{eq:farris}
\dot{\mu}_{xy}
\triangleq \mu_{xy}  + 2\mu_{o1} - \mu_{ox} - \mu_{oy},
\qquad \forall x,y\in \mathcal{X}.
\end{equation}
A classical phylogenetic result (proved for instance
in~\cite[Lemma 7.2.2]{SempleSteel:03}) states that
$\{\dot{\mu}_{xy}\}_{x,y \in\mathcal{X}}$ is an ultrametric 
on $\mathcal{X}$ consistent with the topology of $S$ re-rooted at $o$ and, then,
restricted to $\mathcal{X}$.

In the multi-gene context, however, we cannot apply a Farris transform in this manner. For one, we do not have direct access to the species phylogeny distances $\{\mu_{xy}\}$; rather, we only estimate the gene tree distances $\{\delta_{xy}^{(i)}\}$. Moreover the latter vary across genes according to the MSC. In particular, distance differences (such as those appearing in~\eqref{eq:farris}) are affected by the topology of the gene tree (see Figure~\ref{fig:fixing-top} for an illustration). 
\begin{quote}
\textbf{Key idea 1: To get around this problem, we artificially fix gene tree topologies through conditioning. We also take advantage of the effect of the rooting on the MSC process to avoid using an outgroup.} 
\end{quote}
We give more details on our approach next.  

We turn to the proof of Theorem~\ref{thm:ident}.
We prove the claim for $n=3$. As we discussed, it is straightforward
to extend the proof to $n > 4$. 
Let $S$ be a species phylogeny with three leaves and recall that $r$ is the root of $S$. Unlike the classical Farris transform above, we do not use an outgroup.
Instead, we show how to achieve the same outcome by using only the distribution of $G$ and, in particular, of the random distances $\left\{ \delta_e \right\}_{e\in E_s}$. Notice from~\eqref{eq:farris} that we only need the \emph{differences of distances} between pairs of species in $\mathcal{X} \cup \{r\}$
$$
\Delta_{xy} \triangleq \mu_{rx} - \mu_{ry}.
$$ 
It is these quantities that we derive from the distribution of weighted gene trees. 

The idea is to: 
\begin{enumerate}
\item Condition on an event such that
the rooted topology of a gene tree is guaranteed
to be equal to a fixed, chosen topology. 
Intuitively, we achieve this
by considering
an event where one pair of leaves is ``somewhat
close'' while the other two pairs are
``somewhat far.''

\item Conditioning on this event,
we recover the species-based difference
$\Delta_{xy}=\mu_{rx} - \mu_{ry}$ from the
distribution of gene-based difference $\delta_{xz} - \delta_{yz}$.
Intuitively, letting $w$ be the most recent common
ancestor of $x$ and $y$ on $G$,
when the topology is $xy|z$ then the difference
$\delta_{xz} - \delta_{yz}$ is equal to
$\Delta_{xy}$ irrespective of when  $w$ occurred. See Figure~\ref{fig:fixing-top} for an illustration.
	
\end{enumerate}
\begin{figure*}[h!!]
	\centering
	\includegraphics[scale=3.2]{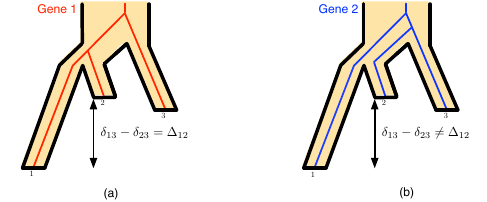}
	\caption{(a) Gene 1 (red gene) has the topology $12 | 3$. Therefore, the gene distance on this gene satisfies the condition that $\delta_{13} - \delta_{23} = \Delta_{12}$. (b) In this case, Gene 2 (blue gene) has the topology $1|23$. Observe that therefore, $\delta_{13} - \delta_{23} \neq \Delta_{12}$. }
	\label{fig:fixing-top}
\end{figure*}
More formally, we establish the following two propositions, whose proofs are in Section~\ref{sec:ident-proofs}.
For 
$x,y \in L$ and $\beta \in [0,1]$,
let $\delta_{xy}^{(\beta)}$ be the
$\beta$-th quantile of $\delta_{xy}$. 
That is, $\delta_{xy}^{(\beta)}$ is the smallest number $\alpha\in [0,1]$ such that $$\mathbb{P}\left[ \delta_{xy} \leq \alpha \right] \geq \beta.$$
Note that this quantile is a function
of the distribution of $G$ (and of the $\vec{\delta}$s).
Our event of interest is defined next.
\begin{proposition}[Fixing the rooted topology of the gene tree]
\label{lem:ident-fixing-top}
Let $(x,y,z)$ be an arbitrary permutation
of $(1,2,3)$.
The event
\begin{equation}
\label{eq:event-ei}
\mathscr{E}_{I}
= \left\{\delta_{xy} \leq \delta_{xy}^{(1/2)},
\delta_{xz} > \delta_{xz}^{(1/2)},
\delta_{yz} > \delta_{yz}^{(1/2)}\right\},
\end{equation}
has positive probability and implies
that the rooted topology of $G$
is $xy|z$.
\end{proposition}
\noindent Conditioning on the event $\mathscr{E}_{I}$,
we then show how to recover the difference
$\Delta_{xy}$ from the
distribution of $\delta_{xz} - \delta_{yz}$.
\begin{proposition}[A formula for the height difference]
\label{lem:ident-height-diff}
Using the notation of Proposition~\ref{lem:ident-fixing-top},
we have
\begin{equation}
\label{eq:ident-height-diff-formula}
\E[\delta_{xz} - \delta_{yz}\,|\,\mathscr{E}_{I}]
= \Delta_{xy},
\end{equation}
almost surely.
\end{proposition}
\noindent Note that the quantity on the l.h.s.~of~\eqref{eq:ident-height-diff-formula}
is a function of the distribution of $G$. From the values of $\Delta_{xy}, x,y \in \mathcal{X}$, we can solve for $\mu_{rx}, x\in \mathcal{X}$. Hence, combining
the properties of the Farris transform with Propositions~\ref{lem:ident-fixing-top} and~\ref{lem:ident-height-diff}, we arrive at Theorem~\ref{thm:ident}.

\subsection{Algorithm~\ref{alg:questAdditiveReduction}: the reduction step}
\label{sec:reduction-step-high-level}

We are now ready to describe the reduction algorithm (Algorithm~\ref{alg:questAdditiveReduction}) and provide guarantees about its behavior. Recall that we are restricting our attention to three leaves $\mathcal{X} = \left\{ 1,2,3 \right\}$ whose species tree topology is $12|3$. The main idea underlying the reduction algorithm is based on the proof of the identifiability result (Theorem~\ref{thm:ident}). That is, we find a set of genes whose topology
is highly likely to be a fixed triplet, 
we estimate the height differences on this set
using the ``sample version'' of~\eqref{eq:ident-height-diff-formula},
and we perform what could be thought of as a ``sequence-based'' Farris
transform. 

Given that we do not have access to the actual gene tree distribution, but only sequence data, there are  several differences with the identifiability proof that make the analysis and the algorithm more involved. A primary challenge is that, in the regime where sequence length is ``short,'' i.e., when $k \ll f^{-2}$, the sequence-based estimates of the gene tree distances are very inaccurate---much less accurate then what is needed for our reduction step to be useful. 
\begin{quote}
\textbf{Key idea 2: To get around this issue, we show how to combine genes satisfying a condition similar to~\eqref{eq:event-ei} to produce a much better estimate of distance differences.}
\end{quote}
We detail the main steps of Algorithm~\ref{alg:questAdditiveReduction} next.

\paragraph{Fixing gene tree topologies.} 
	Here we only have access to sequence data. In particular the $\delta$s are unknown. So, we work instead with the $p$-distances
	$$
	\widehat{p}_{xy}^i = \frac{1}{k}\sum_{j\in [k]}\ind\left\{ \xi^{ij}_x \neq \xi^{ij}_y \right\},
	$$
	for gene $i$ and $x,y \in \mathcal{X}$,
	and their empirical quantiles $\widehat{p}^{(\beta)}_{xy}$.\footnote{Actually, the quantiles
	are estimated from part of the gene set ($\mred{1}$) to avoid
	unwanted correlations. The rest of the analysis is done on the other part.} Similar to  Proposition~\ref{lem:ident-fixing-top},
	we then consider those genes for which
	the event
\begin{equation}
\label{eq:sketch-1}
\left\{\widehat{p}_{xy}^i \leq \widehat{p}_{xy}^{(1/3)}, \widehat{p}_{xz}^{(2/3)}\leq \widehat{p}_{xz}^i\leq \widehat{p}_{xz}^{(5/6)}, \widehat{p}_{yz}^{(2/3)}\leq \widehat{p}_{yz}^i\leq \widehat{p}_{yz}^{(5/6)}\right\},
\end{equation}
	holds for some chosen permutation $(x,y,z)$  
	of $(1,2,3)$. We will call this set of genes $I$. We show that this set has a ``non-trivial'' size and that, with high
	probability, the genes satisfying~\eqref{eq:sketch-1} have topology $xy|z$ (see Proposition~\ref{prop:reduction-fix-top}).\footnote{In fact, the $p$-distances in~\eqref{eq:sketch-1} are estimated over half the gene length to avoid unwanted correlations. That is, we use $\widehat{p}_{xy}^{i\downarrow}$ to compute $I$ (see Step~5 of Algorithm~\ref{alg:questAdditiveReduction}).} In particular, the analysis of this construction accounts for the ``sequence noise'' around the expected values 
\begin{equation}
	p^i_{xy} \triangleq \mathbb{E}\left[ \widehat{p}^i_{xy} \middle | G^{(i)} \right]
	=\frac{3}{4}\left( 1 - e^{-4\delta^i_{xy}/3} \right) \triangleq p(\delta^i_{xy}),\label{eq:definition-p}
\end{equation}
where $p(x) = \frac{3}{4}\left( 1 - e^{-4x/3} \right)$. 

\paragraph{Estimating distance differences.}
Because we work with $p$-distances, we adapt formula~\eqref{eq:ident-height-diff-formula} 
for the difference $\Delta_{xy}$ as follows\footnote{Again, here we use the other half of the sites to avoid correlations with Step 5.}.
Using
$$
\widehat{p}^I_{xz} = \frac{1}{\left| I \right|}\sum_{i\in I}\widehat{p}_{xz}^i
\qquad \text{and} \qquad
\widehat{p}^I_{yz} = \frac{1}{\left| I \right|}\sum_{i\in I}\widehat{p}_{yz}^i,
$$ 
our estimate of the distance differences is given
by
$$
\widehat{\Delta}_{xy} = 
\left\{
-\frac{3}{4}\log\left(1 - \frac{4}{3}\widehat{p}_{xz}^I \right)
\right\}
- 
\left\{
-\frac{3}{4}\log\left(1 - \frac{4}{3}\widehat{p}_{yz}^I \right)
\right\}.
$$
Recall that, for this formula to work, we need to ensure that the topologies of the gene trees used are fixed to be $xy|z$; see Fig.~\ref{fig:fixing-top}, for instance. The logarithmic transforms in the curly
brackets are the usual distance corrections in the Jukes-Cantor 
sequence model (see e.g.~\cite{Steel:16}). 
Note, however, that we perform 
an average over $I$ before the correction; this is important to obtain the correct statistical power of our estimator. A similar phenomenon was leveraged in the METAL algorithm of \cite{dasarathy2015data}. 
The non-trivial
part of the analysis of this step is to bound
the estimation error. Indeed, unlike the identifiability
result, we have a finite amount of gene data and, moreover,
we must account for the sequence noise. This is
done using concentration inequalities in Proposition~\ref{prop:reduction-height-diff}.
\paragraph{Stochastic Farris transform.}
The quantile test of Section~\ref{sec:quantile-test} below is not a
distance-based meth\-od in the traditional
sense of the term. That is, we do not define a pairwise distance matrix on the leaves and use it to deduce the species phylogeny. Instead, our method uses the {\em empirical distribution
of the $p$-distances across genes}. It is for this reason that we do not simply apply the classical Farris 
transform of~\eqref{eq:farris} 
to the estimated distances. 
Rather,  
we perform what we
call a ``stochastic'' Farris transform. That is,
we transform the \emph{sequence data} itself to 
mimic the distribution under an ultrametric species phylogeny.
\begin{quote}
\textbf{Key idea 3: This is done by adding the right amount of noise
	to the sequence data at each gene, as detailed next. 
	It ensures that we properly mimic the contributions from both the multispecies coalescent and the Jukes-Cantor model to the distribution of $p$-distances.}
\end{quote}
 See Algorithm~\ref{alg:questAdditiveReduction} for the full details.	 

For the sake of notational convenience, we will let $\oplus$ denote addition mod-4 and identify ${\tt A,T,G,C}$ with $\{0,1,2,3\}$ in that order when doing this addition. For instance, this means that ${\tt A} \oplus 1 = {\tt T}$ and ${\tt G} \oplus 2 = {\tt A}$. 
\begin{definition}[Stochastic Farris transform]
\label{def:sft}
For a gene $i$, let $\{\xi_x^{i}\}_{x\in \mathcal{X}}$ be a sequence dataset 
over the species $\mathcal{X} = \{1,2,3\}$ and let $\Delta_{xy} = \mu_{rx} - \mu_{ry}, x, y \in \mathcal{X}$. Assume without loss of generality that $\min\{\Delta_{12},\Delta_{13}\} \geq 0$\footnote{This is equivalent to assuming that $\mu_{r1} \geq \max\{\mu_{r2}, \mu_{r3}\}$.}. 
The {\bf stochastic Farris transform} defines a new set of sequences $\{\xi^i_{x,N}\}_{x \in \mathcal{X}}$ such that $\xi_{x,N}^i = \xi_x^i \oplus \epsilon_x^i$, where $\epsilon_x^i\in \{0,1,2,3\}^k$ is an independent random sequence whose $j$-th coordinate is drawn according to 
\begin{align*}
	\epsilon_x^{ij} = \begin{cases}
		0, & \mbox{ w.p. } 1- p(\Delta_{1x}),\\
        1, & \mbox{ w.p. } p(\Delta_{1x})/3,\\
        2, & \mbox{ w.p. } p(\Delta_{1x})/3,\\
        3, & \mbox{ w.p. } p(\Delta_{1x})/3.
	\end{cases}
\end{align*}
We write this as $\{\xi_{x,N}^i\}_{x\in \mathcal{X}} = \mathcal{F}(\{\xi_{x}^i\}_{x\in \mathcal{X}}; \{\Delta_{xy}\}_{x,y\in \mathcal{X}})$. 
\end{definition}
\noindent By the Markov property, for $x,y \in \mathcal{X}$, the ``noisy'' sequence data above satisfy 
\begin{align*}
\mathbb{P}\left[ \xi_{x,N}^i \neq \xi_{y,N}^i \right] &= p \left(\delta_{xy}^i +  \Delta_{1x} + \Delta_{1y}\right)\triangleq r_{xy}^i. 
\end{align*}
Notice that $\delta_{xy}^i$, the random gene tree distance between $x$ and $y$ under gene $i$, can be decomposed as $\mu_{xy} + \Gamma_{xy}^i$, where $\Gamma_{xy}^i$ is the random component contributed by the multispecies coalescent. On the other hand, the set of distances $\mu_{xy} + \Delta_{1x} + \Delta_{1y}$ is ultrametric by the properties of the classical Farris transform. As a result, the stochastic Farris transform modifies the sequence data so that it appears as though it was generated from an ultrametric MSC-JC process. We show this pictorially in Fig.~\ref{fig:sft-works}. 
\begin{figure*}[h!]
  \begin{center}
    \includegraphics[scale = 2.25]{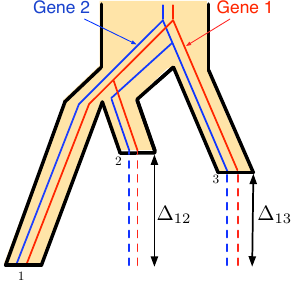}
  \end{center}
  \caption{After the stochastic Farris transform, the leaf edges corresponding to leaves $2$ and $3$, in both Gene 1 (red gene) and Gene 2 (blue gene), are ``stretched'' by $\Delta_{12}$ and $\Delta_{13}$ respectively. As a result, the sequence data appears as though it is drawn from an ultrametric species phylogeny.}
\label{fig:sft-works}
\end{figure*}

In reality, we do not have access to the true differences $\Delta_{xy}, x,y\in \mathcal{X}$. Instead, we employ our estimates $\widehat{\Delta}_{xy}$ for
all $x,y \in \mathcal{X}$ in the previous step to obtain the following
approximate stochastic Farris transform:
\begin{equation}
\label{eq:sketch-reduction-output}
\{\xi_{x,N}^i\}_{x\in \mathcal{X}} 
= \mathcal{F}(\{\xi_x^{i}\}_{x\in \mathcal{X}};\{\widehat{\Delta}_{xy}\}_{x,y\in \mathcal{X}}).
\end{equation}
This is the output of the reduction. See Algorithm~\ref{alg:questAdditiveReduction} for details.
We prove Theorem~\ref{thm:reduction} in Section~\ref{sec:app-redux}. In what follows, we will condition on the implications of Theorem~\ref{thm:reduction} holding.


\section{Concluding remarks}
\label{sec:conclusion}
We have extended the optimal tradeoff (up to log factors) of~\cite{MosselRoch:15} 
beyond the case of equal mutation rates and population sizes.
Several open problems remain:
\begin{enumerate}
	\item Our results assume that the number of leaves $n$ is constant (as $k,m \to \infty$). As $n$ gets larger, the depth of the species phylogeny typically increases. In fact, in the single gene tree reconstruction context, the depth is known to play a critical and intricate role in the data requirement~\cite{ErStSzWa:99a,Mossel:04a,Mossel:07,DaMoRo:11a,DaMoRo:11b}. Understanding the role of the depth under the MSC-JC is an interesting avenue for future work.
	
	\item We have assumed here that the mutation rates are the same across genes. This assumption is not realistic and relaxing it is important for the practical relevance of this line of work. Identifiability issues may arise however~\cite{MatsenSteel:07,Steel:09}. In a related issue, we have assumed, to simplify, that all genes have the same length. (Gene lengths and mutation rates together control the amount of phylogenetic signal in a gene.) We leave for future work how best to take advantage of differing gene lengths (beyond simply truncating to the shortest gene).

	\item A more technical point left open here is to remove the assumption that $k$ grows polynomially with $f$. This may require new ideas.
	
\end{enumerate}


\pagebreak
\bibliographystyle{alpha}
\bibliography{QuEST_refs,own,thesis,RECOMB12,concat}


\pagebreak
\appendix

\section{Models: full definitions}
\label{app:msc-details}
\label{section:models}

In this section, we provide full definitions of the multispecies coalescent and Jukes-Cantor model.

\paragraph{Jukes-Cantor model of sequence evolution}
The Jukes-Cantor model of sequence evolution is detailed in Algorithm~\ref{alg:jukes-cantor}.
\newcounter{tempctr}
\setcounter{tempctr}{\value{figure}}
\renewcommand{\figurename}{Algorithm}
\setcounter{figure}{\thealgorithm}
\begin{figure*}[h!!]
	\begin{framed}
		\vspace{-2mm}
		\begin{enumerate}
			\item Associate to the root $r$ a sequence $\xi_r = \left\{ \xi_{r}^1,\ldots, \xi_{r}^k \right\}\in \left\{{\tt A, T, G, C}\right\}^k$ of length $k$, where each character $\xi_{r}^j$ is drawn independently and uniformly at random from $\left\{{\tt A,T,G,C} \right\}$. 
			\item Initialize the set $U$ with the children of the root $r$ of $G^{(i)}$. 
			\item Repeat until $U= \emptyset$. 
			\begin{enumerate}
				\item Pick $u\in U$, and let $u^-$ be the parent of $u$ in $G^{(i)}$.
				\item Associate a sequence $\xi_u \in \left\{ {\tt A,T,G,C} \right\}^k$ as follows. $\xi_u$ is obtained from $\xi_{u^-}$ by mutating each site independently with probability $p_{(u,u^-)}$. If a mutation occurs at a site $j$, it gets assigned a uniformly random character from $\left\{ {\tt A,T,G,C} \right\}$, else the corresponding character from $\xi_{u^-}$ simply gets copied.
				\item Remove $u$ from $U$ and add any children of $u$ to $U$. 
			\end{enumerate}
		\end{enumerate}
	\end{framed}
	\caption{The Jukes-Cantor process}
	\label{alg:jukes-cantor}
\end{figure*}
\renewcommand{\figurename}{Figure}
\setcounter{algorithm}{\thefigure}
\setcounter{figure}{\thetempctr}

\paragraph{Multispecies coalescent}\newcounter{tempctrctr}
\setcounter{tempctrctr}{\value{figure}}
\renewcommand{\figurename}{Algorithm}
\setcounter{figure}{\thealgorithm}
\begin{figure*}[h!!]
	\begin{framed}
		\noindent \underline{Function \msc{}($\currentvertex{}$)}.
		\begin{enumerate}
			\itemsep-0.5em
			\item {\bf If }$\currentvertex{} \in L$, i.e., $\currentvertex$ is a leaf of $S$:
			\vspace{-0.5em}
			\begin{enumerate}
				\itemsep0em
				\item {\bf Return} the following single edge (root-extended) tree: $(\currentvertex{},r')$. One vertex of this edge corresponds to the current leaf and the other is an ancestor to this leaf, and the root of the tree. The length of the edge created is $\mu_e\times \tau_e$, where $e\in E_s$ is the edge incident upon $\currentvertex{}$ in $S$. 
			\end{enumerate}
			\item {\bf Else}
			\vspace{-0.5em}
			\begin{enumerate}
				\itemsep0em
				\item Let $d_1$ and $d_2$ be the descendants of $\currentvertex{}$. 
				\item If $\currentvertex{}$ is $r$, the root of $S$, then set $\tau,\mu = \infty$. Otherwise, let $\tau$ and $\mu$ respectively be the length  and mutation rate of the branch connecting $\currentvertex{}$ to its immediate ancestor in $S$.   
				\item {\bf Return} the following forest: \mergeforest{}(\msc{}($d_1$), \msc{}($d_2$), $\tau$, $\mu$)
			\end{enumerate}
		\end{enumerate}
		
		\noindent \underline{Function \mergeforest{}($F_1$, $F_2$, $\tau$, $\mu$).} 
		\begin{enumerate}
			\itemsep0em
			\item Create a new forest $F$ that is a union of $F_1$ and $F_2$. Set $k =$ number of roots (or lineages) in $F$.
			\item {\bf While } $k>1$:
			\vspace{-0.5em}
			\begin{enumerate}
				\item Choose a random pair of (distinct) roots $r_1$ and $r_2$ from $F$. Also draw a random time $t\sim$ Exp$({k\choose 2})$. 
				\item {\bf If $t\geq \tau$}
				\begin{enumerate}
					\itemsep0em
					\item Increase the length of  all the $k$ root edges in $F$ by $\mu\times \tau$. {\bf Return} $F$. 
				\end{enumerate}
				\item {\bf Else }
				\begin{enumerate}
					\itemsep0em
					\item Set $\tau := \tau-t$. Create new vertices $r_3, r_3'$. 
					\item Make $r_1$ and $r_2$ descendants of $r_3$, where the lengths of the branches ($r_1$ , $r_3$) and ($r_2$, $r_3$) are both set equal to $\mu\times t$. Make $r_3'$ the root of this newly created tree, connecting $r_3'$ to $r_3$ with a length $0$ branch. 
					\item Also, add $\mu\times t$ to the root edges of the other trees in $F$.  Now, $F$ has one fewer root, so set $k := k-1$.
				\end{enumerate}
			\end{enumerate}
			\item If $k$ is 1, then {\bf return} $F$, adding $\mu\times \tau$ to the unique root edge of $F$. 
		\end{enumerate}
	\end{framed}
	\caption{The multispecies coalescent process}
	\label{alg:msc-process}
\end{figure*}
\renewcommand{\figurename}{Figure}
\setcounter{algorithm}{\thefigure}
\setcounter{figure}{\thetempctrctr}
Let $S = (V_s, E_s; r, \vec{\tau}, \vec{\mu})$ be a fixed species phylogeny. For the sake of this algorithmic description, we will work with what we call \emph{root-extended} trees and forests. Given a weighted rooted tree, the corresponding root-extended tree simply has a new vertex $r'$ that is connected to $r$ with a (potentially zero-) weighted edge. We will call $r'$ the root of such a tree, and the edge connecting $r$ and $r'$ as the \emph{root edge}. A root-extended forest is simply a union of root-extended trees. Only for  the description that follows, when we write tree and forest, we mean the root-extended versions unless otherwise specified. 


We first describe a function \msc{} that takes as input a vertex $\currentvertex{} \in V_s$, and returns a forest. We will obtain a random gene tree from the multispecies coalescent as follows: 
(1) invoke the function \msc{} with $r$, the root of $S$ as input, and (2) contract the root-edge of the tree returned (thus making it a gene tree per Definition~\ref{def:gene-tree}). That is, for $i\in [m]$, $G^{(i)}= $  \msc{}$(r)$, with the root edge contracted. This function, as we can see in Algorithm~\ref{alg:msc-process},  recursively descends the species phylogeny $S$ and it calls the function \mergeforest{}() at every stage of this recursion. 

The \mergeforest{}() function works with rooted forests. This function operates at each branch of the species tree, and (potentially) merges the genealogies of its two descendant populations. It takes as input two forests $F_1$ and $F_2$ corresponding to the descendants of the current branch (or population) that it is invoked at. It also takes the mutation rate $\mu$ and length $\tau$ associated with the current branch. It then returns a single forest $F$ after performing (potentially) multiple coalescence operations. The details are in Algorithm~\ref{alg:msc-process}.

Fig.~\ref{fig:msc-example-alg} shows two sample draws from the multispecies coalescent process. Notice that while the topology of Gene 1 (red gene) agrees with the topology of the underlying species tree, the topology of Gene 2 (blue gene) does not. This happens since in Step~2(b) of \mergeforest{}, if the randomly drawn time $t$ is larger than $\tau$, the chosen pair of lineages do not coalesce in the current population. This sort of discordance in the topologies of gene trees is called \emph{incomplete lineage sorting} (ILS).
For more details, we refer the reader e.g.~to~\cite{DegnanRosenberg:09}.

The density of the likelihood of a gene tree ${G}^{(i)} = \left({V}^{(i)},{E}^{(i)}\right)$ can now be written down as follows. We will focus our attention on the branch $e\in E$ of the species tree and for the gene tree $\gcal^{(i)}$, let $I^{(i)}_e$ and $O^{(i)}_e$ be the number of lineages entering and leaving the branch $e$ respectively. For instance, consider Gene 1 in Figure~\ref{fig:mscExample2}. 
\begin{figure}[t!]
	\centering
	\includegraphics[scale = 2]{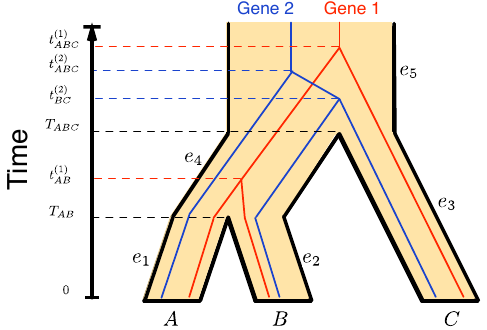}
	\caption{A species phylogeny (the thick, shaded tree) and two samples from the multispecies coalescent.}
	\label{fig:mscExample2}
\end{figure}
Here, two lineages enter the branch $e_4$  and one lineage leaves it. On the other hand, in the case of Gene~2 in Figure~\ref{fig:mscExample2}, two lineages enter the branch $e_4$ and two lineages leave it. Let $t_{e,s}^{(i)}, s = \left\{1,2,\ldots,I_e^{(i)} - O_{e}^{(i)}+1\right\}$ be the $s-$th coalescent time corresponding to $\gcal^{(i)}$ in the branch $e$. From the algorithm described above, each pair of lineages in a population can coalesce at a random time drawn according to the Exp$(1)$ distribution  independently of each other. Therefore, after the $(s-1)$-th coalescent event at time $t_{e,s-1}^{(i)}$, there are $I_e^{(i)} - s+1$ surviving lineages in branch $e$ and the likelihood that the $s-$th coalescence time in branch $e$ is $t_{e,s}^{(i)}$ corresponds to the event that the minimum of ${I_e^{(i)} - s+1\choose 2}$ random variables distributed according to Exp$(1)$ has the value $t_{e,s}^{(i)} - t_{e,s-1}^{(i)}$. It follows that the density of the likelihood of $\gcal^{(i)}$ can be written as 
\begin{equation}
\prod_{e\in E}\prod_{s=1}^{I^{(i)}_e-O_e^{(i)}+1}\exp\left\{-{I_e^{(i)} -s+1\choose 2}\left[t_{e,s}^{(i)} - t_{e,s-1}^{(i)}\right]\right\},
\end{equation} 
where, for convenience, we let $t^{(i)}_{e,0}$ and $t^{(i)}_{e,I^{(i)}_{e}-O^{(i)}_{e}+1}$ be respectively the divergence times of the population in $e$ and of its parent population. 

We will also need the density and quantiles of gene tree pairwise distances.
For a pair of leaves $a,b \in L$, $\delta_{ab}$ is the branch length induced by the random gene tree which is drawn according to the MSC. Notice that, by definition, $\delta_{ab}\geq \mu_{ab}$ . We will let $f_{ab}\left(\cdot\right)$ and $F_{ab}\left(\cdot\right)$ denote respectively the density and the cumulative density function of the random variable $Z_{ab}\triangleq\frac{\delta_{ab}-\mu_{ab}}{2}$. 
Because of the memoryless property of the exponential, 
it is natural to think of the distribution of $Z_{ab}$ as a mixture of distributions whose supports are disjoint, corresponding to the different branches that the lineages go through before coalescing.
We state this more generally as follows.
	Suppose that $U_1,U_2,\ldots, U_r$ are subsets of $\mathbb{R}$ such that they satisfy 
	\begin{equation*}
	\sup(U_i)\leq \inf(U_{i+1}),\;\;\; i =1,2,\ldots, r.
	\end{equation*}  
	Suppose that $f$ is a probability density function such that 
	\begin{equation*}
	f(x) = \sum_{i=1}^r \omega_i f_i(x),
	\end{equation*}
	where $\omega_1,\omega_2,\ldots,\omega_r\in (0,1)$ are such that $\sum_{i=1}^r \omega_i = 1$, and the density $f_i$ is supported on $S_i$ for $i=1,\ldots,r$. Then, the quantile function of $f$ is given as follows
	\begin{equation*}
	\mathbb{Q}_F(\alpha) \triangleq\inf\left\{x\in \mathbb{R}: \alpha\leq F(x)\right\}= \sum_{i=1}^r\ind\left\{\alpha \in \left[\sum_{j=0}^{i-1} \omega_j,\sum_{j=0}^{i} \omega_j\right)\right\}\mathbb{Q}_{F_i}\left(\frac{\alpha - \sum_{j=0}^{i-1} \omega_i}{\omega_{i}}\right),
	\end{equation*}  
	where, we set $\omega_0 = 0$.
Specializing to $f_{ab}$ under the multispecies coalescent, it follows that there exists a finite sequence of constants $\mu_1,\ldots,\mu_r \in [\mu_L,\mu_U]$ and $h_0,\ldots,h_{r-1} \in [f',g'+ng]$ such that 
$$
\omega_i
= e^{-\sum_{j=1}^{i-1}\mu_{j}^{-1}(h_j - h_{j-1})}
- e^{-\sum_{j=1}^{i}\mu_{j}^{-1}(h_j - h_{j-1})},
$$
and
$$
f_i(x)
=
\frac{\mu_i^{-1} e^{-\mu_i^{-1}(x - h_{i-1})}}
{1 - e^{-\mu_i^{-1}(h_i - h_{i-1})}}.
$$
Cancellations lead to
\begin{equation*}
f_{ab}(x) = \sum_{i=1}^r e^{-\sum_{j=1}^{i-1}\mu_{j}^{-1}(h_j - h_{j-1})}\mu_i^{-1} e^{-\mu_i^{-1} (x - h_{i-1})}.
\end{equation*}
This formula implies that the density is bounded between positive constants. We will need the following implication.
For any $\alpha\in [0,1)$, we let $\delta_{ab}^{(\alpha)}$ and $p_{ab}^{(\alpha)}$ denote  the $\alpha$-quantile of the $\delta_{ab}$ and $p_{ab}$ respectively. Since by definition $p_{ab} = p(\delta_{ab})$, we have that $p^{(\alpha)}_{ab} = p(\delta_{ab}^{(\alpha)})$, where $p(x) = \frac{3}{4}\left( 1 - e^{-4x/3} \right)$. 
Then, for any $0 < \beta' < \beta < 1$, there are constants $0 < c' < c'' <+\infty$ (depending on $\mu_L, \mu_U, g, g', n, \beta', \beta$) such that
for any $\xi\in (0,1-\beta')$, we have
\begin{equation}
\label{lem:mscQuantile-delta}
c' \xi \leq 
\delta_{ab}^{(\beta+\xi)} - \delta_{ab}^{(\beta)}
\leq
c'' \xi,
\end{equation}
and, hence,
\begin{equation}
\label{lem:mscQuantile}
c' \xi \leq 
p_{ab}^{(\beta+\xi)} - p_{ab}^{(\beta)}
\leq
c'' \xi.
\end{equation}

\section{Identifiability result: proofs}
\label{sec:ident-proofs}

The key steps in the proof of Theorem~\ref{thm:ident} follow.

\bigskip

\begin{prevproof}{Proposition}{lem:ident-fixing-top}
	Recall the definition of the event 
	$$
	\mathscr{E}_{I}
	= \left\{\delta_{xy} \leq \delta_{xy}^{(1/2)},
	\delta_{xz} > \delta_{xz}^{(1/2)},
	\delta_{yz} > \delta_{yz}^{(1/2)}\right\}.
	$$
	Our goal is to show that it has positive probability and that it implies
	that the rooted topology of $G$
	is $xy|z$. This makes sense on an intuitive level because the
	event $\mathscr{E}_{I}$ requires
	that $\delta_{xy}$ is somewhat small and that
	$\delta_{xz}$,
	$\delta_{yz}$ are somewhat large. To make this rigorous
	we use the fact that, conditioned on coalescence occurring
	in the top population, the time to coalescence inside that
	population is identically distributed for all pairs of lineages. That observation facilitates the comparison
	the $\delta$-quantiles, as we show now.
	
	Let $\Gamma_{xy}$ be twice the weighted height of the MSC process between the lineages of $x$ and $y$ \textit{in the common ancestral population}. Notice that with this definition of $\Gamma_{xy}$, we can write $\delta_{xy} = \mu_{xy} + \Gamma_{xy}$. 
	We let $\Gamma_{xy} = 0$ if the coalescence between
	the lineages of $x$ and $y$ occurs
	\textit{below} the common ancestral population (which in this
	case is only possible for the two closest populations in the species tree), an event we denote by $\mathscr{B}_{xy}$. 
	For $\beta \in [0,1]$, let $\Gamma_{xy}^{(\beta)}$
	be the $\beta$-th quantile of $\Gamma_{xy}$. We define the quantities above similary for the other pairs. We make three
	observations:
	\begin{enumerate}
		\item[a)] By definition of $\Gamma_{xy}$, the event
		$\{\delta_{xy} \leq \delta_{xy}^{(1/2)}\}$ implies the event
		$\{\Gamma_{xy} \leq \Gamma_{xy}^{(1/2)}\}$. (Note that the
		two events are not in fact equivalent, though, because of the possibility
		that $\Gamma_{xy}^{(1/2)} = 0$.) Similarly,
		$\{\delta_{xz} > \delta_{xz}^{(1/2)}\}$ implies
		$\{\Gamma_{xz} > \Gamma_{xz}^{(1/2)}\}$
		and $\{\delta_{yz} > \delta_{yz}^{(1/2)}\}$ implies
		$\{\Gamma_{yz} > \Gamma_{yz}^{(1/2)}\}$.
		
		\item[b)] Irrespective of the
		species tree topology, the lineages of at least of one the pairs
		$(x,z)$ or $(y,z)$ can only coalesce in the common
		ancestral population. See Figures~\ref{fig:msc-example-alg}~and~\ref{fig:fixing-top}, for instance.  
		
		\item[c)] By symmetry, conditioned on coalescence in the common ancestral population, all $\Gamma$s are equal in distribution, i.e.,
		$$
		\Gamma_{xy}\,|\,\mathscr{B}_{xy}^{c}
		\stackrel{\rm d}{=}
		\Gamma_{xz}\,|\,\mathscr{B}_{xz}^{c}
		\stackrel{\rm d}{=}
		\Gamma_{yz}\,|\,\mathscr{B}_{yz}^{c}.
		$$
	\end{enumerate}
	
	As a consequence of b) and c), we have
	\begin{equation}
	\label{eq:ident-proof-1}
	\Gamma_{xy}^{(1/2)} \leq \max\{\Gamma_{xz}^{(1/2)}, \Gamma_{yz}^{(1/2)}\},
	\end{equation}
	where we used that, conditioned on $\mathscr{B}_{xy}$, it holds that $\Gamma_{xy} = 0$. Combining this 
	with a), we get that $\mathscr{E}_{I}$ implies
	the event 
	$$
	\left\{\Gamma_{xy} \leq \Gamma_{xy}^{(1/2)},
	\Gamma_{xz} > \Gamma_{xz}^{(1/2)},
	\Gamma_{yz} > \Gamma_{yz}^{(1/2)}\right\},
	$$
	which together with~\eqref{eq:ident-proof-1}
	implies the event
	$$
	\Gamma_{xy} < \max\{\Gamma_{xz},\Gamma_{yz}\}.
	$$
	This last event can only happen when the gene tree topology
	is $xy|z$.
	
	It remains to prove that $\mathscr{E}_{I}$
	has positive probability. Let $\mathscr{F}$ be the event
	that $G$ has rooted topology $xy|z$. 
	Note that
	\begin{eqnarray*}
		\P[\mathscr{E}_{I}]
		&\geq& \P[\mathscr{E}_{I},\mathscr{F}]\\
		&=& \P[\delta_{xy} \leq \delta_{xy}^{(1/2)}]\,
		\P[\mathscr{F}\,|\,\delta_{xy} \leq \delta_{xy}^{(1/2)}]\,
		\P[\delta_{xz} > \delta_{xz}^{(1/2)},
		\delta_{yz} > \delta_{yz}^{(1/2)}\,|\,\mathscr{F},\delta_{xy} \leq \delta_{xy}^{(1/2)}],
	\end{eqnarray*}
	where each term on the last line is clearly
	positive under the MSC. 
\end{prevproof}

\begin{prevproof}{Proposition}{lem:ident-height-diff}
	Conditioned on $\mathscr{E}_{I}$, we know 
	from Lemma~\ref{lem:ident-fixing-top} that the coalescence between the lineages of $x$ and $z$ happens in the common ancestral population of $x,y$ and $z$, irrespective of the species tree topology. The same
	holds for $y$ and $z$. This implies that 
	\begin{equation}
	\label{eq:ident-proof-2}
	\delta_{xz} = \mu_{rx} + \mu_{rz} + \Gamma_{xz},
	\end{equation}
	and
	\begin{equation}
	\label{eq:ident-proof-3}
	\delta_{yz} = \mu_{ry} + \mu_{rz} + \Gamma_{yz},
	\end{equation}
	where the $\Gamma$s are defined in the proof of Lemma~\ref{lem:ident-fixing-top}. 
	Observe further that in fact, conditioned on $\mathscr{E}_{I}$, 
	\begin{equation}
	\label{eq:ident-proof-4}
	\Gamma_{xz} = \Gamma_{yz},
	\end{equation}
	almost surely.
	Hence, combining~\eqref{eq:ident-proof-2},~\eqref{eq:ident-proof-3}, and~\eqref{eq:ident-proof-4},
	$$
	\E[\delta_{xz} - \delta_{yz}\,|\,\mathscr{E}_{I}]
	= \mu_{rx} - \mu_{ry} = \Delta_{xy},
	$$
	as claimed.
\end{prevproof}

\section{Main theorem: proof}
\label{sec:main-proof}
\label{sec:algo-high-level-proof}

Theorem~\ref{thm:main} follows from Proposition~\ref{thm:main-full} below.
\begin{proposition}[Data requirement: general version]\label{thm:main-full}
	Suppose we have data $\left\{ \xi^{ij} \right\}_{i\in [m], j\in [k]}$ generated according to the MSC-JC$(m,k)$ process on a species phylogeny $S = (V_s, E_s; r, \vec{\tau},\vec{\mu})$. The mutation rates, leaf-edge lengths and internal-edge lengths are respectively in $(\mu_L, \mu_U)$, $(f', g')$ and $(f, g)$. For any $\ve>0$ and $C>0$ there is a constant $\const{1}>0$ such that Algorithm~\ref{alg:questAdditiveQuantile} correctly identifies the topology of $S$ restricted to $\mathcal{X} = \{1,2,3\}$ with probability at least $1 - \ve$, provided there is a partitioning of the set of genes $[m] = \mred{1}\sqcup \mred2 \sqcup \mquant1 \sqcup \mquant{2}$ such that the following conditions hold: 
	
	\noindent\begin{minipage}{.5\textwidth}
		\begin{align*}
		\left| \mathcal{M}_{{\rm R}1} \right| &\geq \const{1} \log \ve^{-1}\phantom{\left(\const{2}\frac{\log k}{k f^2} \right)}\hspace{-6mm}\\
		\left| \mathcal{M}_{{\rm Q}1} \right| &\geq \const{1} \alpha^{-1} \log \ve^{-1}
		\end{align*}	
		
	\end{minipage}%
	\begin{minipage}{.5\textwidth}
		\begin{align*}
		\begin{aligned}
		\left| \mathcal{M}_{{\rm R}2} \right| &\geq \const{1}\left( 1\vee \frac{\log k}{k f^2} \right) \log \ve^{-1}\\
		\left| \mathcal{M}_{{\rm Q}2} \right| &\geq \const{1} f^{-2}\left( \alpha + f \right)\log \ve^{-1}, 
		\end{aligned}
		\end{align*}
		
	\end{minipage}
	\vspace{5mm}
	
	\noindent where $\alpha = m^{-1} \log m \vee k^{-0.5} \sqrt{\log k}$. And, the sequence length $k$ satisfies: 
	\begin{align*}
	k &\geq \const{1} \log \left(\left| \mathcal{M}_{{\rm R}2} \right|\ve^{-1}\right) \vee \const{1}\left( f^{-1}\sqrt{\log k}\right)^{C}.
	\end{align*}
\end{proposition}
\noindent Under condition~\eqref{eq:main-condition} of Theorem~\ref{thm:main}, to satisfy the inequalities above,
we can choose $|\mred1|  \gtrsim 
1$, $|\mred2| \gtrsim 
1 \lor \frac{1}{k f^2}$, $|\mquant{1}| \gtrsim 
1$ and $|\mquant{2}| \gtrsim \frac{1}{f} \lor \frac{1}{\sqrt{k} f^2}$.

\bigskip

\begin{prevproof}{Proposition}{thm:main-full}
Without loss of generality we assume that the topology of the true species tree restricted to this triple is $12|3$.
	
	\paragraph{Main steps} Let $S_\xcal$ be the species 
	tree $S$ restricted to $\xcal = \{1,2,3\}$ and let
	$r'$ denote its root, i.e., the most recent common ancestor of $\mathcal{X}$.  
	We let $\alpha = \max\left\{ m^{-1} \log m, k^{-0.5}\sqrt{\log k} \right\}$. 
	We partition the loci $[m] = \mred{1} \sqcup \mred{2}\sqcup \mquant{1} \sqcup \mquant{2}$, such that the size of each partition satisfies the conditions specified in Proposition~\ref{thm:main-full}.  
	The reconstruction algorithm
	on $\xcal$ has two steps: 1) a reduction to the ultrametric
	case by the addition of noise and 2) the ultrametric quantile test.
	We divide the proof accordingly:
	\begin{enumerate}
		\item[1)] {\it Ultrametric reduction:} In this step, we invoke Algorithm~\ref{alg:questAdditiveReduction} with sequence data $\{\xi_{x}^{ij}: x\in \mathcal{X}, i \in [m], j\in [k]\}$. The algorithm outputs new sequences $\{\xi_{x,N}^{ij}:x\in \mathcal{X}, i \in \mquant{1} \sqcup \mquant{2}, j\in [k]\}$, and Proposition~\ref{thm:reduction-full} (proved in Section~\ref{sec:app-redux}) guarantees that these output sequences appear as though they were drawn from an almost-ultrametric species phylogeny $S_{\mathcal{X}}'$ that has the same topology as $\mathcal{S}_{\mathcal{X}}$. In particular, using Proposition~\ref{thm:reduction-full}, we know that there is a constant $\const{1} >0$ such that  if 
	\begin{align*}
&|\mred1|, |\mred2| 
\geq \const{1} \log(4 \eps^{-1}),\\
&k 
\geq \const{1} \log |\mred2|
+ \const{1} \log(4 \eps^{-1}),\\
&k |\mred2|
\geq \const{1} \phi^{-2} \log(4 \eps^{-1}).
\end{align*}
		then, $\{\xi^{ij}_{x,N} : x\in \mathcal{X}, i\in \mquant{1}\sqcup \mquant{2}, j\in [k]\}$ has the same distribution as a multispecies coalescent process on  $S_{\mathcal{X}}'$ with branch lengths $(\hat{\dot{\mu}}_{xy})$, where $S_{\mathcal{X}}'$ and $S$ have the same rooted topology, and $(\hat{\dot{\mu}}_{xy})$ and $({\dot{\mu}}_{xy})$ are $\mathcal{O}(f /\sqrt{\log k})$-close with probability at least $ 1- \ve$. Notice that we have set the value of $\phi$ in Proposition~\ref{thm:reduction-full} to $\mathcal{O}(f /\sqrt{\log k})$, which will turn out to be what we need in Step 2 below.
		
		\item[2)] {\it Quantile test:} Now, we invoke Algorithm~\ref{alg:questAdditiveQuantile} with the sequence data $\{\xi_{x,N}^{ij}: i \in \mquant{1} \sqcup \mquant{2}, j\in [k]\}$ output by Step~1. 
		By Propositions~\ref{prop:quantileBehavior} and~\ref{prop:sample-quantile-test} (proved in Section~\ref{sec:app-quantile}), it follows that
		if
		\begin{align*}
		\left| \mquant{1} \right| &\geq \const{1} \alpha^{-1} \log \ve^{-1}\\
		\left| \mquant{2} \right| &\geq \const{1} f^{-2}\left( \alpha + f \right)\log \ve^{-1},
		\end{align*}
		then, with probability at least $ 1- \ve$, Algorithm~\ref{alg:questAdditiveQuantile} returns the right topology . 
	\end{enumerate}
	This concludes the proof. 
\end{prevproof}

\section{Ultrametric reduction: proofs}
\label{sec:app-redux}

Theorem~\ref{thm:reduction} follows from Proposition~\ref{thm:reduction-full} below.
In Proposition~\ref{thm:reduction-full}, we show that the approximate stochastic Farris transform defined in Section~\ref{sec:reduction-step-high-level} 
outputs sequence data that looks statistically as though 
it was generated from an ultrametric species phylogeny. 

Given estimates $\widehat{\Delta}_{xy}$ for all $x,y \in \mathcal{X}$, and supposing that $\min\{\widehat{\Delta}_{12}, \widehat{\Delta}_{13}\}\geq 0$ (the other cases follow similarly), we let
$$
\widehat{\dot{\mu}}_{xy}
= \mu_{xy} + \widehat{\Delta}_{1x} + \widehat{\Delta}_{1y},
\qquad x,y\in \mathcal{X}.
$$
Compare this to the definition of $\dot{\mu}_{xy}$ 
in~\eqref{eq:farris}.
Recall Definition~\ref{def:species-tree+metric}, and let $S' = (V_s, E_s, r, \vec{\tau}, \vec{\hat{\dot{\mu}}})$ be a species phylogeny with the same topology and branch lengths as $S$ restricted to $\mathcal{X}$, and mutation rates $\{\widehat{\dot{\mu}}\}_{e\in E_s}$ that are chosen such that: (a) $\widehat{\dot{\mu}}_e = \mu_e$ if $e\in E_s$ is an internal branch, and (b) for all $e\in E_s$ that are incident on the leaves of $S$, let $\widehat{\dot{\mu}}_e$ be chosen (uniquely)  such that mutation rate weighted distance on $S'$ between any pair of leaves $x, y \in \mathcal{X}$ is given by $\widehat{\dot{\mu}}_{xy}$. The sets $\mred1, \mred2$ referred to below are defined in Algorithm~\ref{alg:questAdditiveReduction}.
\begin{proposition}[Ultrametric reduction: general version]
	\label{thm:reduction-full}
	Suppose that we have sequence data $\left\{ \xi^{ij} \right\}_{i\in [m], j\in [k]}$ generated according to the MSC-JC$(m,k)$ process on a three-species phylogeny $S = (V_s, E_s; r, \vec{\tau},\vec{\mu})$. The mutation rates, leaf-edge lengths and internal-edge lengths are respectively in $(\mu_L, \mu_U)$, $(f', g')$ and $(f, g)$. 	
	Then, the output
	of Algorithm~\ref{alg:questAdditiveReduction}
	is distributed according to the MSC-JC process on the species tree
	$S'$ defined above. 
	Furthermore there is a constant $\const{2} > 0$ such that, 
	for any $\eps, \phi \in (0,1)$, 
	with probability at least $1-\eps$, 
	$\widehat{\dot{\mu}}_{xy}$
	is $\phi$-close
	to the ultrametric $(\dot{\mu}_{xy})$,
	provided
	\begin{align*}
	&|\mred1|, |\mred2| 
	\geq \const{2} \log(4 \eps^{-1}),\\
	&k 
	\geq \const{2} \log |\mred2|
	+ \const{2} \log(4 \eps^{-1}),\\
	&k |\mred2|
	\geq \const{2} \phi^{-2} \log(4 \eps^{-1}).
	\end{align*}
\end{proposition}
\noindent Under condition~\eqref{eq:reduction-condition} of Theorem~\ref{thm:reduction}, to satisfy the inequalities above,
we can choose $|\mred1|  \gtrsim 
1$ and $|\mred2| \gtrsim 
1 \lor \frac{1}{k f^2}$.

\subsection{Proof of Proposition~\ref{thm:reduction-full}}
\label{sec:proof-reduction-theorem}

\begin{prevproof}{Proposition}{thm:reduction-full}
	As explained in Section~\ref{sec:reduction-step-high-level}, 
	there are three main steps to this proof,
	which we summarize in a series of propositions.	
	
	For a gene $i$ and leaves $x,y \in \mathcal{X}$,
	let
	$$
	\widehat{p}_{xy}^i = \frac{1}{k}\sum_{j\in [k]}\ind\left\{ \xi^{ij}_x \neq \xi^{ij}_y \right\}.
	$$
	Furthermore, we need to split the above average into two halves to avoid unwanted correlations as we explain below. We denote these as\footnote{For simplicity, we assume that $k$ is even. This is not a critical requirement and can be easily relaxed.} 
	\begin{align*}
	\widehat{p}^{i\downarrow}_{xy} = \frac{2}{k}\sum_{j=1}^{k/2}\ind\{\xi^{ij}_x \neq \xi^{ij}_y\} \qquad \mbox{ and} \qquad \widehat{p}^{i\downarrow}_{xy} = \frac{2}{k}\sum_{j=k/2+1}^{k}\ind\{\xi^{ij}_x \neq \xi^{ij}_y\}.
	\end{align*} 
	And, for $\beta \in [0,1]$, let
	$\widehat{p}^{(\beta)}_{xy}$ 
	be the corresponding empirical quantiles computed based on the set $\{\widehat{q}_{xy}^i: i\in \mred{1}\}$.  
	Fix a permutation $(x,y,z)$  
	of $(1,2,3)$. 
	Consider the following subset of genes in $\mred{2}$: 
	$$I =
	\left\{i\in \mred{2}\,:\,\widehat{p}^{i\downarrow}_{xy} \leq \widehat{p}_{xy}^{(1/3)},
	\widehat{p}_{xz}^{(2/3)}
	\leq \widehat{p}^{i\downarrow}_{xz} \leq \widehat{p}_{xz}^{(5/6)},
	\widehat{p}_{yz}^{(2/3)}
	\leq 
	\widehat{p}^{i\downarrow}_{yz} 
	\leq \widehat{p}_{yz}^{(5/6)}\right\}.
	$$
	
	We first show that the rooted topologies in $I$
	are highly likely to be $xy|z$. We also 
	prove some technical claims that will be useful in 
	the proof of Proposition~\ref{prop:reduction-height-diff}
	below.
	\begin{proposition}[Fixing gene tree topologies]
		\label{prop:reduction-fix-top}
		There are constants $\const{9}, \const{10}, \const{10}', \eps_0 > 0$
		such that, with probability at least
		$$
		1 
		- 10\exp(-2\const{9}^2 |\mred1|) 
		- 6  |\mred2| \exp\left(
		- k \eps_0^2
		\right)
		- 2\exp\left(-2 \const{10}^2 |\mred2|\right),
		$$
		the following hold:
		\begin{enumerate}
			\item[(a)] the rooted topology of all gene trees
			in $I$ is $xy|z$,
			
			\item[(b)] for all $i\in I$, $p^i_{xy} \leq p^{(7/24)}_{xy}$,
			$p^{(17/24)}_{xz} \leq p^i_{xz} \leq p^{(19/24)}_{xz}$,
			$p^{(17/24)}_{yz} \leq p^i_{yz} \leq p^{(19/24)}_{yz}$,
			
			\item[(c)] the size of $I$ is greater than $\const{10}' |\mred2|$.
		\end{enumerate}
		
	\end{proposition}
	The proof is given in Section~\ref{sec:reduction-fix-top}. Notice that Proposition~\ref{prop:reduction-fix-top} guarantees that all the genes in $I$ satisfy conditions (a) and (b). We can weaken our requirements on how big $k$ needs to be by relaxing this and performing a more careful analysis.

	Using
	$$
	\widehat{p}^I_{xz} = \frac{1}{\left| I \right|}\sum_{i\in I}\widehat{p}_{xz}^{i\uparrow}
	\qquad \text{and} \qquad
	\widehat{p}^I_{yz} = \frac{1}{\left| I \right|}\sum_{i\in I}\widehat{p}_{yz}^{i\uparrow},
	$$ 
	let
	$$
	\widehat{\Delta}_{xy} = 
	\left\{
	-\frac{3}{4}\log\left(1 - \frac{4}{3}\widehat{p}_{yz}^I \right)
	\right\}
	- 
	\left\{
	-\frac{3}{4}\log\left(1 - \frac{4}{3}\widehat{p}_{xz}^I \right)
	\right\}.
	$$
	Recall that
	$$
	\Delta_{xy} = \mu_{rx} - \mu_{ry}.
	$$
	We next show that $\widehat{\Delta}_{xy}$
	is a good approximation to $\Delta_{xy}$.
	Let $\mathscr{I}$ be the event that the conclusion
	of Proposition~\ref{prop:reduction-fix-top} holds. 
	To simplify the notation,
	throughout this proof, we use $\PP$ and $\EE$
	to denote the probability and expectation operators
	{\em conditioned on $\mathscr{I}$}.
	\begin{proposition}[Estimating distance differences]
		\label{prop:reduction-height-diff}
		There is a constant $\const{11} \in (0,1)$ such that
		with $\PP$-probability at least
		$$
		1 - 4\exp\left( - \const{11} k|\mred2| \phi^2\right),
		$$
		the following holds:
		$$
		\left|
		\widehat{\Delta}_{xy}
		- \Delta_{xy}
		\right| 
		\leq
		\phi/2.
		$$
	\end{proposition}
	The proof is in Section~\ref{sec:reduction-height-diff}.
	
	We repeat the height difference estimation above for all
	pairs in $\mathcal{X}$.  Therefore, by a union bound, we get the above guarantee for all pairs with probability at least $1 - 12\exp\left( - \const{11} k|\mred2| \phi^2\right)$. 
	Without loss of generality, assume that
	$$
	\mu_{r1} \geq \max\{\mu_{r2}, \mu_{r3}\}, 
	$$
	and recall the Farris transform
	$$
	\dot{\mu}_{xy}
	= \mu_{xy} + 2\mu_{r1}- \mu_{rx} - \mu_{ry} = \mu_{xy} + \Delta_{1x} + \Delta_{1y},
	\qquad x, y \in \mathcal{X},
	$$
	which defines an ultrametric, and consider the
	approximation
	$$
	\widehat{\dot{\mu}}_{xy}
	= \mu_{xy} + \widehat{\Delta}_{1x} + \widehat{\Delta}_{1y},
	\qquad x, y\in \mathcal{X}.
	$$
	Assuming that the conclusion of Proposition~\ref{prop:reduction-height-diff}
	holds for
	all $x,y \in \mathcal{X}$, we have shown that
	$(\widehat{\dot{\mu}}_{xy})$ is $\phi$-close to 
	the ultrametric $(\dot{\mu}_{xy})$.
	As we explained in Section~\ref{sec:reduction-step-high-level},
	we produce a new sequence dataset using an
	approximate stochastic Farris transform
	$$
	\{\xi_{x,N}^i\} 
	= \mathcal{F}(\{\xi_x^{i}\};(\widehat{\mu}_{rx})).
	$$
	By the Markov property, the transformation
	$\mathcal{F}$ has the effect of stretching
	the leaf edges of the gene trees by the 
	appropriate amount. 
	
	Hence, again by a union bound, we get the claim
	of Theorem~\ref{thm:reduction} except with
	probability
	\begin{equation}
	\label{eq:reduction-prob-error}
	10\exp(-2\const{9}^2 |\mred1|) 
	+ 6  |\mred2| \exp\left(
	- 2 k \eps_0^2
	\right)
	+ 2\exp\left(-2 \const{10}^2 |\mred2|\right)
	+ 12\exp\left( - \const{11} k|\mred2| \phi^2\right).
	\end{equation}
	We can get the data requirement result by asking for the conditions under which the above quantity is less than $\ve$. 
\end{prevproof}

\subsubsection{Proof of Proposition~\ref{prop:reduction-fix-top}}
\label{sec:reduction-fix-top}

\begin{prevproof}{Proposition}{prop:reduction-fix-top}
	Let $(x,y,z)$ be an arbitrary permutation
	of the leaves $(1,2,3)$.
	The idea of the proof is to rely on
	Proposition~\ref{lem:ident-fixing-top},
	which we rephrase in terms of $p$-distances. For a gene $G_i$
	Let
	$$
	p^i_{xy} 
	= \frac{3}{4}\left( 1 - e^{-4\delta^i_{xy}/3} \right). 
	$$
	And, for $\beta \in [0,1]$, the corresponding $\beta$-th
	quantile is given by 
	$$
	p^{(\beta)}_{xy} 
	= \frac{3}{4}\left( 1 - e^{-4 \delta^{(\beta)}_{xy}/3} \right);
	$$
	similarly for the other pairs.
	Then, by Proposition~\ref{lem:ident-fixing-top},
	the event
	$$
	\mathscr{E}^i_{I}
	= \left\{p^i_{xy} \leq p_{xy}^{(1/2)},
	p_{xz}^{(1/2)} < p^i_{xz} ,
	p_{yz}^{(1/2)} < p^i_{yz} \right\},
	$$
	implies
	that the rooted topology of $G_i$
	is $xy|z$.
	Our goal is to show that
	\begin{equation}
	\label{eq:qi}
	\mathscr{Q}_i
	= \left\{\widehat{p}^{i\downarrow}_{xy} \leq \widehat{p}_{xy}^{(1/3)},
	\widehat{p}_{xz}^{(2/3)}
	\leq \widehat{p}^{i\downarrow}_{xz} \leq \widehat{p}_{xz}^{(5/6)},
	\widehat{p}_{yz}^{(2/3)}
	\leq 
	\widehat{p}^{i\downarrow}_{yz} 
	\leq \widehat{p}_{yz}^{(5/6)}\right\},
	\end{equation}
	implies $\mathscr{E}^i_{I}$ with high probability. We do this by controlling the deviations of $\widehat{p}_{uw}^{(\beta)}$ and $\widehat{p}^{i\downarrow}_{uw}$. We state the necessary
	claims as a series of lemmas.
	(The upper bounds on $\widehat{p}^{i\downarrow}_{xz}$
	and $\widehat{p}^{i\downarrow}_{yz}$ in~\eqref{eq:qi}
	are included for technical reasons that will
	be explained in the proof of Proposition~\ref{prop:reduction-height-diff}.
	This requirement may not be needed, but it 
	makes the analysis simpler.)
	
	Recall that we use only the genes in $\mred{}$ for the reduction step and this in turn is divided into disjoint subsets $\mred{1}$ and $\mred{2}$.
	The quantiles are estimated using $\mred{1}$,
	while $\mred{2}$ is used to compute $I$.
	We do {\it not} argue about the deviation of
	$\widehat{p}_{uw}^{(\beta)}$ from the {\it true}
	$\beta$-th quantile of the distribution
	of $\widehat{p}_{uw}^i$. Instead
	we show that
	$\widehat{p}_{uw}^{(\beta)}$ is
	close to the $\beta$-th quantile $p_{uw}^{(\beta)}$
	of the disagreement probability {\em under the MSC},
	that is, the quantile
	{\it without the sequence noise}.
	We argue this way because
	the events that we are ultimately interested in
	(whether a certain coalescence event
	has occured in a particular population)
	are expressed in terms of the MSC. Note that, in order to obtain
	a useful bound of this type, 
	we must assume that the sequence
	length is sufficiently long, that is, that the
	sequence noise is reasonably small. Hence
	this is one of the steps of our argument where
	we require a lower bound on $k$. 
	
	\begin{lemma}[Deviation of $\widehat{p}_{uw}^{(\beta)}$]
		\label{lem:reduction-deviation-quantile}
		Fix a pair $u,w \in \mathcal{X}$ and a constant $\beta \in (0,1)$.
		For all $\eps_0 > 0$ and $0 < \eps_1 < \min\{\beta,1-\beta\}$, there is a constant
		$c_0 > 0$ such that
		$$
		\P\left[p_{uw}^{(\beta-\eps_1)} - \eps_0 
		\leq \widehat{p}_{uw}^{(\beta)} 
		\leq 
		p_{uw}^{(\beta + \eps_1)} + \eps_0 \right] 
		\geq
		1 - 
		2 \exp\left(
		- 2 \const{9}^2 |\mred{1}|
		\right),
		$$
		provided that $k$ is greater than a constant
		depending on $\eps_0$ and $\eps_1$.
	\end{lemma}
	\begin{proof}
		We prove one side of the first equation. The other inequalities follow 
		similarly. Define the random variable
		\begin{align*}
		M &= \left| \left\{ i \in \mred{1}: \widehat{p}^{i\downarrow}_{uw}\leq p^{(\beta + \eps_1)}_{uw} + \eps_0 \right\} \right|,
		\end{align*}
		and observe that 
		\begin{align*}
		\P\left[ 
		\widehat{p}^{(\beta)}_{uw} > p^{(\beta + \eps_1)}_{uw} + \eps_0 
		\right] 
		&\leq \P\left[ M <  \beta |\mred{1}| \right].
		\end{align*}
		To bound the probability on the r.h.s., we note that
		\begin{eqnarray*}
			\P\left[
			\widehat{p}^{i\downarrow}_{uw}
			\leq p^{(\beta + \eps_1)}_{uw} + \eps_0
			\right]
			&\geq& 
			\P\left[
			\widehat{p}^{i\downarrow}_{uw}
			\leq p^{(\beta+\eps_1)}_{uw} + \eps_0
			\,\middle|\, p^i_{uw} \leq p^{(\beta+\eps_1)}_{uw}\right]
			\P\left[
			p^i_{uw} \leq p^{(\beta+\eps_1)}_{uw}
			\right]\\
			&\geq&
			\left[
			1 - \exp\left(
			-k \eps_0^2
			\right)
			\right]
			\left(
			\beta
			+ \eps_1
			\right),
		\end{eqnarray*}
		by Hoeffding's inequality~\cite{hoeffding1963probability} and
		the definition of $p^{(\beta
			+ \eps_1)}_{uw}$. We also
		used that 
		$\E[\widehat{p}^{i\downarrow}_{uw}\,|\,p_{uw}] = p_{uw}$.
		By Hoeffding's inequality applied to $M$,
		we have that
		$$
		\P\left[ 
		\widehat{p}^{(\beta)}_{uw} > p^{(\beta + \eps_1)}_{uw} + \eps_0 
		\right] 
		\leq 
		\P\left[ M <  \beta |\mred{1}| \right]
		\leq \exp\left(
		- 2 \const{9}^2 |\mred{1}|
		\right),
		$$
		where
		$$
		c_0 
		= 
		\left(
		\beta
		+ \eps_1
		\right)
		\left[
		1 - \exp\left(
		- k \eps_0^2
		\right)
		\right]
		-
		\beta,
		$$
		which is strictly positive,
		provided that
		$k$ is greater than a constant
		depending on $\eps_0$ and $\eps_1$.
	\end{proof}
	
	On the other hand, standard concentration inequalities allow
	us to control the deviation of $\widehat{p}^i_{uw}$.
	Observe that, $p^i_{uw}$ being itself random,
	the deviation holds conditionally on the value
	of $p^i_{uw}$.
	\begin{lemma}[Deviation of $\widehat{p}^i_{uw}$]
		\label{lem:reduction-deviation-hatp}
		Fix a pair $u, w \in \mathcal{X}$.
		For all $i$ and $\eps_0 > 0$,
		$$
		\P\left[
		|\widehat{p}^{i\downarrow}_{uw}
		- p^i_{uw}|
		\geq \eps_0
		\,|\, p^i_{uw}
		\right]
		\leq 
		2 \exp\left(
		-  k \eps_0^2
		\right),
		$$
		almost surely.
	\end{lemma}
	
	\begin{proof}
		Note that, conditioned on $p^i_{uw}$,
		$k/2 \widehat{p}^{i\downarrow}_{uw}$ is distributed as
		$\mathrm{Bin}(k,p^i_{uw})$.
		The result then follows from Hoeffding's inequality.
	\end{proof}
	Fix $0 < \eps_1 < 1/24$ and pick $\eps_0 > 0$ small
	enough that 
	\begin{align}
	&p_{xy}^{(7/24)}
	\leq 
	p_{xy}^{(1/3-\eps_1)} - 2\eps_0 
	\leq 
	p_{xy}^{(1/3 + \eps_1)} + 2\eps_0 
	\leq
	p^{(9/24)}_{xy}\nonumber\\
	&p^{(15/24)}_{xz}
	\leq
	p_{xz}^{(2/3-\eps_1)} - 2\eps_0 
	\leq 
	p_{xz}^{(2/3 + \eps_1)} + 2\eps_0 
	\leq
	p^{(17/24)}_{xz}\nonumber\\
	&p^{(19/24)}_{xz}
	\leq
	p_{xz}^{(5/6-\eps_1)} - 2\eps_0 
	\leq 
	p_{xz}^{(5/6 + \eps_1)} + 2\eps_0 
	\leq
	p^{(21/24)}_{xz}\label{eq:reduction-proof-quantiles}\\
	&p^{(15/24)}_{yz}
	\leq
	p_{yz}^{(2/3-\eps_1)} - 2\eps_0 
	\leq 
	p_{yz}^{(2/3 + \eps_1)} + 2\eps_0 
	\leq
	p^{(17/24)}_{yz}\nonumber\\
	&p^{(19/24)}_{yz}
	\leq
	p_{yz}^{(5/6-\eps_1)} - 2\eps_0 
	\leq 
	p_{yz}^{(5/6 + \eps_1)} + 2\eps_0 
	\leq
	p^{(21/24)}_{yz}.\nonumber
	\end{align}
Notice that the fact that these inequalities hold is guaranteed by~\eqref{lem:mscQuantile} (in Section~\ref{section:models}) which characterizes the behavior of the quantile functions of the random variables associated with the MSC. 
	
	Let $\mathscr{E}_{\mathrm{qu}}$ be the event that
	the inequality in Lemma~\ref{lem:reduction-deviation-quantile},
	i.e., $p_{uw}^{(\beta-\eps_1)} - \eps_0 
	\leq \widehat{p}_{uw}^{(\beta)} 
	\leq 
	p_{uw}^{(\beta + \eps_1)} + \eps_0$,
	holds for $\widehat{p}^{(1/3)}_{xy}$,
	$\widehat{p}^{(2/3)}_{xz}$, $\widehat{p}^{(5/6)}_{xz}$,
	$\widehat{p}^{(2/3)}_{yz}$ and 
	$\widehat{p}^{(5/6)}_{yz}$, which occurs with probability
	at least $1 - 10\exp(-2 c_0^2 |\mred{1}|)$
	by a union bound.
	Let $\mathscr{D}_i$ be the event that
	the inequality in Lemma~\ref{lem:reduction-deviation-hatp},
	i.e., $|\widehat{p}^{i\downarrow}_{uw}
	- p^i_{uw}|
	\geq \eps_0$, holds
	for all pairs $(u,w)$ in $\mathcal{X}$, 
	an event which occurs with probability
	at least $1 - 6\exp(-k \eps_0^2)$
	by a union bound.
	Given $\mathscr{E}_{\mathrm{qu}}$, $\mathscr{D}_i$ and $\mathscr{Q}_i$,
	we have
	$$
	p^i_{xy}
	\leq 
	\widehat{p}^i_{xy} + \eps_0
	\leq 
	\widehat{p}_{xy}^{(1/3)} + \eps_0
	\leq p^{(1/3+\eps_1)}_{xy} + 2\eps_0
	\leq p^{(1/2)}_{xy},
	$$
	and similarly for the other pairs. That is,
	$\mathscr{E}^i_{I}$ holds.
	Finally, we bound the probability that
	all $i$ in $I$ satisfy $\mathscr{D}_i$
	with the probability that all $i$ in $\mred2$
	satisfy $\mathscr{D}_i$. (In fact we show below
	that with high probablity $|I| = \Theta(|\mred2|)$.)
	That is, the probability
	that all $i\in I$ satisfy $\mathscr{E}^i_{I}$
	{\em simultaneously} is at least
	\begin{equation}
	\label{eq:reduction-fix-top-final-1}
	\P[\mathscr{E}^i_{I},\forall i \in I]
	\geq 
	1 
	- 10\exp(-2 c_0^2 |\mred1|) 
	- 6  |\mred2| \exp\left(
	- k \eps_0^2
	\right).
	\end{equation}
	
	It remains to bound the size of $I$.
	\begin{lemma}[Size of $I$]
		\label{lem:reduction-size-i}
		There are constants $\const{10}, \const{10}'>0$ such that
		\begin{equation}
		\label{eq:reduction-fix-top-final-2}
		\P[|I| \geq \const{10}' |\mred2|\,|\,\mathscr{E}_{\mathrm{qu}}]
		\geq 1- 2\exp\left(-2 \const{10}^2 |\mred2|\right),
		\end{equation}
		provided $k$ is greater than
		a constant depending on $\eps_0$.
	\end{lemma}
	\begin{proof}
		We show that, under $ \mathscr{E}_{\mathrm{qu}}$,
		the event $\mathscr{Q}_i$ has constant probability
		and we apply Hoeffding's inequality.
		
		Observe that, by~\eqref{eq:reduction-proof-quantiles},
		the events $\{p^i_{xy} \leq p^{(7/24)}_{xy}\}$
		and $\mathscr{D}_i$ imply
		$$
		\widehat{p}^i_{xy}
		\leq p^i_{xy} + \eps_0
		\leq p^{(7/24)}_{xy} + \eps_0
		\leq p^{(1/3-\eps_1)}_{xy} - \eps_0
		\leq \widehat{p}^{(1/3)}_{xy}.
		$$
		Hence, a similar argument shows that
		$$
		\mathscr{D}_i
		\cap
		\{p^i_{xy} \leq p^{(7/24)}_{xy},
		p^{(17/24)}_{xz} \leq p^i_{xz} \leq p^{(19/24)}_{xz},
		p^{(17/24)}_{yz} \leq p^i_{yz} \leq p^{(19/24)}_{yz}
		\},
		$$
		implies $\mathscr{Q}_i$. This leads to the following
		lower bound
		\begin{eqnarray*}
			&&\P[\mathscr{Q}_i \,|\, \mathscr{E}_{\mathrm{qu}}]\\
			&&\quad\geq \P[\mathscr{D}_i
			\cap
			\{p^i_{xy} \leq p^{(7/24)}_{xy},
			p^{(17/24)}_{xz} \leq p^i_{xz} \leq p^{(19/24)}_{xz},
			p^{(17/24)}_{yz} \leq p^i_{yz} \leq p^{(19/24)}_{yz}
			\}\,|\, \mathscr{E}_{\mathrm{qu}}]\\
			&&\quad\geq 
			\P[
			p^i_{xy} \leq p^{(7/24)}_{xy},
			p^{(17/24)}_{xz} \leq p^i_{xz} \leq p^{(19/24)}_{xz},
			p^{(17/24)}_{yz} \leq p^i_{yz} \leq p^{(19/24)}_{yz}]\\
			&&\quad \qquad \times \P[\mathscr{D}_i
			\,|\,
			\{p^i_{xy} \leq p^{(7/24)}_{xy},
			p^{(17/24)}_{xz} \leq p^i_{xz} \leq p^{(19/24)}_{xz},
			p^{(17/24)}_{yz} \leq p^i_{yz} \leq p^{(19/24)}_{yz}
			\} \cap \mathscr{E}_{\mathrm{qu}}]\\
			&&\quad\geq \const{10}''
			\left[
			1 -
			6 \exp\left(
			- k \eps_0^2
			\right)
			\right],
		\end{eqnarray*}
		for some constant $\const{10}'' > 0$. 
		This existence of the latter constant follows from an argument similar to that leading up to~\eqref{lem:mscQuantile} (but is somewhat complicated by the fact that $p^i_{xy}$, $p^i_{xz}$
		and $p^i_{yz}$ are not independent).
		The expression on the last line is a strictly
		positive constant provided $k$ is greater than
		a constant depending on $\eps_0$.
		Finally, applying Hoeffding's inequality to
		$|I|$, we get the result.
	\end{proof}
	Combining~\eqref{eq:reduction-fix-top-final-1} 
	and~\eqref{eq:reduction-fix-top-final-2}
	concludes the proof.
\end{prevproof}

\subsubsection{Proof of Proposition~\ref{prop:reduction-height-diff}}
\label{sec:reduction-height-diff}

\begin{prevproof}{Proposition}{prop:reduction-height-diff}
	Fix $x,y \in \xcal$ and let $z$ be the unique element in $\xcal - \{x,y\}$. 
	The proof idea is based on Proposition~\ref{lem:ident-height-diff}.
	Recall that $\mathscr{I}$ be the event that the conclusion
	of Proposition~\ref{prop:reduction-fix-top} holds. 
	Let also $\mathscr{G}_I$ be the event that the weighted
	gene trees in $I$ are $\left\{ G_i \right\}_{i\in I}$.
	Similarly to the proof of Proposition~\ref{lem:ident-height-diff}
	we note that,
	conditioned on $\mathscr{I}$, in all genes in $I$ the coalescences between the lineages of $x$ and $z$ happen in the common ancestral population of $x$, $y$ and $z$, irrespective of the species tree topology. The same
	holds for $y$ and $z$. That implies that,
	for $i\in I$, 
	\begin{equation*}
	\delta^i_{xz} = \mu_{rx} + \mu_{rz} + \Gamma^i_{xz},
	\end{equation*}
	and
	\begin{equation*}
	\delta^i_{yz} = \mu_{ry} + \mu_{rz} + \Gamma^i_{yz},
	\end{equation*}
	where the $\Gamma^i$s are defined as in the proof of Lemma~\ref{lem:ident-fixing-top}. 
	Observe further that in fact, conditioned on $\mathscr{I}$, for $i\in I$
	\begin{equation}
	\label{eq:ident-proof-4-prime}
	\Gamma^i_{xz} = \Gamma^i_{yz},
	\end{equation}
	almost surely.
	Hence 
	\begin{align*}
	\EE\left[ \widehat{p}^I_{xz} \,\middle|\, \mathscr{G}_I \right] 
	&= \EE\left[ \frac{1}{\left| I \right|}\sum_{i\in I} \widehat{p}^i_{xz} \,\middle|\, \mathscr{G}_I\right]\nonumber\\ 
	&= \frac{1}{\left| I \right|}\sum_{i\in I} p^i_{xz}\nonumber\\ 
	&= \frac{3}{4}\left( 1 - \frac{1}{\left| I \right|}\sum_{i\in I}e^{-4 \delta_{xz}^i /3 } \right)\nonumber\\
	&= \frac{3}{4}\left( 1 - 2 e^{-4 \mu_{rx}/3 - 4 \mu_{rz}/3}\left( \frac{1}{\left| I \right|}\sum_{i\in I}e^{-4 \Gamma^i_{xz}/3} \right) \right),\label{eq:hd-cond-exp}
	\end{align*}
	and similarly for the pair $(y,z)$.
	Letting $\ell(x) = -\frac{3}{4}\log\left( 1- \frac{4}{3}x\right)$,
	we get
	\begin{align*}
	\ell\left(\EE\left[ \widehat{p}_{xz}^I \,\middle|\, \mathscr{G}_I \right]\right)
	- 
	\ell\left(\EE\left[ \widehat{p}_{yz}^I \,\middle|\, \mathscr{G}_I \right]\right)
	&=
	-\frac{3}{4}\log\left(
	\frac{1 - 4/3 \EE\left[ \widehat{p}_{xz}^I \,\middle|\, \mathscr{G}_I \right]}{1 - 4/3 \EE\left[ \widehat{p}_{yz}^I \,\middle|\, \mathscr{G}_I \right]}
	\right)\nonumber\\
	&=
	-\frac{3}{4}\log\left(
	\frac{e^{-4 \mu_{rx}/3-4 \mu_{rz}/3}\left( \frac{1}{\left| I \right|} \sum_{i\in I} e^{-4  \Gamma_{xz}^i /3 }\right)}{e^{-4 \mu_{ry}/3-4 \mu_{rz}/3}\left( \frac{1}{\left| I \right|} \sum_{i\in I} e^{-4  \Gamma_{yz}^i /3 }\right)}
	\right)\nonumber\\
	&=
	-\frac{3}{4}\log\left(
	e^{-4 \mu_{rx}/3 + 4 \mu_{ry}/3}
	\right)\nonumber\\
	&= \Delta_{xy},\nonumber
	\end{align*}
	where we used~\eqref{eq:ident-proof-4-prime}
	on the third line.
	Observe that the computation above
	relies crucially on the conditioning on $\mathscr{G}_I$.

	It remains to bound the deviation
	of 
	$$
	\widehat{\Delta}_{xy} = 
	\ell\left(\widehat{p}_{xz}^I\right)
	- 
	\ell\left(\widehat{p}_{yz}^I\right),
	$$ 
	around
	$$
	\Delta_{xy}
	= \ell\left(\EE\left[ \widehat{p}_{xz}^I \,\middle|\, \mathscr{G}_I \right]\right)
	- 
	\ell\left(\EE\left[ \widehat{p}_{yz}^I \,\middle|\, \mathscr{G}_I \right]\right),
	$$
	and take expectations with respect to $\mathscr{G}_I$.
	We do this by controlling the error on
	$\widehat{p}^I_{xz}$ and
	$\widehat{p}^I_{yz}$, conditionally on $\mathscr{G}_I$.
	Indeed, observe that the function $\ell$ satisfies the following Lipschitz property: for $0 \leq x \leq y \leq M < 1/2$,
	\begin{equation}
	\label{eq:hd-lipschitz}
	\left|\ell(x) - \ell(y) \right| = \int_x^y \frac{1}{1-4t/3}\,{\rm d}t\leq \frac{\left| x-y \right|}{1-4M/3}.
	\end{equation}
	Hence, 
	to control
	$|\widehat{\Delta}_{xy}
	- \Delta_{xy}|$,
	it suffices to bound 
	$\left|\widehat{p}_{uz}^I - \EE\left[ \widehat{p}_{uz}^I \,\middle|\, \mathscr{G}_I \right]\right|$
	and 
	$\max\left\{\widehat{p}_{uz}^I, \EE\left[ \widehat{p}_{uz}^I \,\middle|\, \mathscr{G}_I \right]\right\}$
	for $u = x,y$.
	
	To bound $\EE\left[ \widehat{p}_{uz}^I \,\middle|\, \mathscr{G}_I \right]$,
	we use the upper bounds on $\widehat{p}^i_{xz}$
	and $\widehat{p}^i_{yz}$ in the definition of the
	set $I$.
	\begin{lemma}[Conditional expectation of $\widehat{p}_{uz}^I$]
		\label{lem:hd-max-hatpI}
		Fix $u = x$ or $y$. There is a constant
		$\const{12}' \in (0,1/2)$ small enough,
		$$
		\EE\left[ \widehat{p}_{uz}^I \,\middle|\, \mathscr{G}_I \right]
		\leq \frac{1}{2} - \const{12}',
		$$
		$\PP$-almost surely.
	\end{lemma}
	\begin{proof}
		Using
		$p_{uz}^i = \EE\left[ \widehat{p}_{uz}^i \,\middle|\, \mathscr{G}_I \right]$ for $i \in I$, 
		by
		Proposition~\ref{prop:reduction-fix-top} (b),
		we have that
		\begin{align*}
		\EE\left[ \widehat{p}_{uz}^i \,\middle|\, \mathscr{G}_I \right]
		&= p^i_{uz}
		= \frac{4}{3}\left( 1 - e^{-4\delta_{uz}^i/3} \right)
		\leq \frac{1}{2} -\const{12}',
		\end{align*}
		for some constant $c_2' \in (0,1/2)$. Again, this constant depends on bounds on the mutation rate and the depth of the tree. 
		Hence,
		\begin{align*}
		\EE\left[ \widehat{p}^I_{uz} \,\middle|\, \mathscr{G}_I \right] 
		&= \frac{1}{\left| I \right|}\sum_{i\in I} \EE\left[ \widehat{p}_{uz}^i \,\middle|\, \mathscr{G}_I \right]
		\leq \frac{1}{2} - c_2'.
		\end{align*}
	\end{proof}
	To bound $\left|\widehat{p}_{uz}^I - \EE\left[ \widehat{p}_{uz}^I \,\middle|\, \mathscr{G}_I \right]\right|$,
	we use Hoeffding's inequality.
	\begin{lemma}[Conditional deviation of $\widehat{p}_{uz}^I$]
		\label{lem:hd-dev-hatpI}
		Fix $u=x$ or $y$. For all $\phi' > 0$,
		$$
		\PP\left[
		\left|\widehat{p}_{uz}^I - \EE\left[ \widehat{p}_{uz}^I \,\middle|\, \	{G}_I \right]\right|
		\geq \phi'
		\,\middle|\, \mathscr{G}_I \right]
		\leq 
		2\exp\left( -k|I|(\phi')^2\right),
		$$
		almost surely.
	\end{lemma}
	\begin{proof}
		Observe first that, {\em conditioned on $\mathscr{G}_I$}, 
		the $k \left| I \right|$ sites 
		that are averaged over in the computation of 
		\begin{equation}
		\label{eq:hatp-i-def}
		\widehat{p}^I_{uz}
		= \frac{2}{k|I|}\sum_{i\in I} \sum_{j=k/2+1}^k\ind\left\{ \xi^{ij}_u \neq \xi^{ij}_z \right\},
		\end{equation}
		are {\em independent}. Secondly, each random variable 
		in~\eqref{eq:hatp-i-def} is bounded by 1. Therefore, from Hoeffding's inequality, we have that 
		\begin{equation*}
		\PP\left[
		\left|\widehat{p}_{uz}^I - \EE\left[ \widehat{p}_{uz}^I \,\middle|\, \mathscr{G}_I \right]\right|
		\geq \phi'
		\,\middle|\, \mathscr{G}_I \right]
		\leq 2\exp\left( -k|I|(\phi')^2\right),
		\end{equation*}
		almost surely.
	\end{proof}
	
	We set $\phi' = \frac{1}{2} \const{12}' (\phi/2)$, which
	is $< \const{12}'$
	since $\phi \leq 1$.
	Combining~\eqref{eq:hd-lipschitz}
	and Lemmas~\ref{lem:hd-max-hatpI} and~\ref{lem:hd-dev-hatpI}, we get
	that conditioned on $\mathscr{G}_I$
	\begin{align*}
	\left|
	\widehat{\Delta}_{xy}
	- \Delta_{xy}
	\right|
	&\leq 
	\left|
	\ell\left(\widehat{p}_{xz}^I\right)
	- \ell\left(\EE\left[ \widehat{p}_{xz}^I \,\middle|\, \mathscr{G}_I \right]\right)
	\right|
	+
	\left|
	\ell\left(\widehat{p}_{yz}^I\right)
	- \ell\left(\EE\left[ \widehat{p}_{yz}^I \,\middle|\, \mathscr{G}_I \right]\right)
	\right|\\
	&\leq \frac{\left|\widehat{p}_{xz}^I - \EE\left[ \widehat{p}_{xz}^I \,\middle|\, \mathscr{G}_I \right]\right|
	}{1-4/3\max\left\{\widehat{p}_{xz}^I, \EE\left[ \widehat{p}_{xz}^I \,\middle|\, \mathscr{G}_I \right]\right\}}
	+
	\frac{\left|\widehat{p}_{yz}^I - \EE\left[ \widehat{p}_{yz}^I \,\middle|\, \mathscr{G}_I \right]\right|
	}{1-4/3\max\left\{\widehat{p}_{yz}^I, \EE\left[ \widehat{p}_{yz}^I \,\middle|\, \mathscr{G}_I \right]\right\}}\\
	&\leq 2 \frac{\phi'}{2 (\const{12}'-\phi')}\\
	&\leq 2 \frac{\frac{1}{2} \const{12}' (\phi/2)}{2 (\const{12}'-\frac{1}{2} \const{12}' (\phi/2))}\\
	&\leq \phi/2,
	\end{align*}
	where we used that $\phi/2 \leq 1$,
	except with $\PP$-probability
	$$
	4\exp\left( -k|I|(\phi')^2\right)
	\leq 
	4\exp\left( - \frac{1}{8} (\const{12}')^2 \const{10}' k|\mred2| \phi^2\right)
	=
	4\exp\left( - \const{12} k|\mred2| \phi^2\right),
	$$
	by setting $c_2 = \frac{1}{8} (c_2')^2 c_1'$.
	Taking expectations with respect to
	$\mathscr{G}_I$ gives the result.
\end{prevproof}

\section{Quantile test: robustness analysis}
\label{sec:app-quantile}
\label{sec:quantile-test}

In this section, we analyze Algorithm~\ref{alg:questAdditiveQuantile}.

\paragraph{Control of empirical quantiles} To perform the quantile test, Algorithm~\ref{alg:questAdditiveQuantile} has access to a set of genes $\mquant{}$ that were not used in the reduction step above; this is to avoid unwanted correlations. This set is in turn partitioned as $\mquant{} = \mquant{1}\sqcup \mquant{2}$ so that $\left| \mquant{1} \right|, \left| \mquant{2} \right|$ satisfy the conditions of Proposition~\ref{thm:main-full}. The first step in Algorithm~\ref{alg:questAdditiveQuantile} is to compute a well-chosen empirical quantile of $\widehat{q}_{xy}$ for each pair of leaves $x,y\in \mathcal{X}$ based on the dataset $\left\{ \widehat{q}^i_{xy}: i\in \mquant{1} \right\}$. The quantile we compute is (a constant multiple of)
$$\alpha = \max\left\{ m^{-1}\log m, k^{-0.5}\sqrt{\log k} \right\}.$$
In the following proposition, we show that these empirical quantiles are well-behaved, and provide a good estimate of the $\alpha$-quantile of the underlying MSC random variables. We define the random variables $q_{xy}^i$ and $r_{xy}^i$ associated to a gene tree $i$: 
\begin{align*}
q_{xy}^i &= p (\delta_{xy}^i + \widehat{\Delta}_{1x} + \widehat{\Delta}_{1y}),\\
r_{xy}^i &= p (\delta_{xy}^i + {\Delta}_{1x} + {\Delta}_{1y}).
\end{align*}
Also, we need the $0$-th quantile of these random variables. Notice that 
$$q_{xy}^{(0)} = p\left( \delta_{xy}^{(0)} + \widehat{\Delta}_{1x} + \widehat{\Delta}_{1y} \right) = p (\mu_{xy} + \widehat{\Delta}_{1x} + \widehat{\Delta}_{1y}).$$ And similarly, $$r_{xy}^{(0)} = p (\mu_{xy} + {\Delta}_{1x} + {\Delta}_{1y}).$$ Finally, $\phi>0$ is the closeness parameter from Proposition~\ref{thm:reduction-full}. 

\begin{proposition}[Quantile behaviour]
	\label{prop:quantileBehavior}
	Let $\alpha = \max\left\{ m^{-1}\log m, k^{-0.5}\sqrt{\log k} \right\}$. Then, there exist constants $\const{3}, \const{4}, \const{5} >0$ such that, for each  pair of leaves $x,y \in \mathcal{X}$, that the $\const{3}\alpha$-quantile satisfies the following
	\begin{equation*}
	\widehat{q}_{xy}^{(\const{3}\alpha)}\in \left[ {q}_{xy}^{(0)}, {q}_{xy}^{(0)}+ \const{5} \alpha \right]\subset \left[{r}_{xy}^{(0)}-\const{5}\phi, {r}_{xy}^{(0)} +\const{5}\phi + \const{5}\alpha\right]
	\end{equation*}
	with probability at least $1 - 6\exp\left( -\const{4} \left| \mquant{1} \right| \alpha \right)$, provided we  condition on the implications of Proposition~\ref{thm:reduction-full} holding.   
\end{proposition}
\noindent We prove this proposition in Section~\ref{sec:proofOfQuantileBehaviorLemma}.  In what follows, we will let $\bar{\mathbb{P}}$ and $\bar{\mathbb{E}}$ denote the probability and expectation measures conditioned on the event that implications of Propositions~\ref{thm:reduction-full} and~\ref{prop:quantileBehavior}. 

\paragraph{Expected version of quantile test} Let $\widehat{q}_\ast$ denote the maximum among $\left\{ \widehat{q}^{(\const{3} \alpha)}_{xy}: x,y\in \mathcal{X} \right\}$. We  use the genes in $\mquant{2}$ (which are not affected by the conditioning under $\bar{\mathbb{P}}$) to define a similarity measure among pairs of leaves in $\mathcal{X}$: 
\begin{equation}
\widehat{s}_{xy} = \frac{1}{\left| \mquant{2} \right|}\left| \left\{ i\in \mquant{2}: \widehat{q}^i_{xy} \leq \widehat{q}_\ast \right\} \right|.
\end{equation}
We next show that this similarity measure has the right behavior in expectation. That is, defining $$s_{xy} \triangleq \bar{\mathbb{E}}\left[ \widehat{s}_{xy} \right],$$ which is the expected version of our similarity measure, we show that $s_{12} > \max\{s_{13}, s_{23}\}$. This means that the $s_{xy}$s expose the topology of the tree $S_\mathcal{X}$. 
\begin{proposition}[Expected version of quantile test]
	\label{prop:quantileTestPop}
	Let $s_{xy}$ be as defined above. Then for any $C_2 >0$, there exist constants $\const{6}, \const{7}>0$ such that 
	$s_{12} - \max\{s_{13}, s_{23}\} \geq \const{6} p(3f/4) > 0$ provided
	\begin{align*}
	&m \geq \const{7}\frac{1}{p(3f/4)} \log\left( \frac{1}{p(3f/4)} \right)\\
	&k \geq \const{7}\left( \frac{\sqrt{\log k}}{p(3f/4)} \right)^{1/C_2},
	\end{align*}
	and the closeness parameter $\phi \in \mathcal{O}(p(3f/4)/\sqrt{\log k})$.
\end{proposition}
\noindent The proof is given in Section~\ref{sec:pop-quantile-test-proof}.

\paragraph{Sample version of quantile test} Finally, we conclude the proof by demonstrating that the empirical versions of the similarity measures defined above are also consistent with the underlying species tree topology with high probability. This follows from a concentration argument detailed in Section~\ref{sec:proof-sample-quantile-test}. 
\begin{proposition}[Sample version of quantile test]
	\label{prop:sample-quantile-test}
	There exists a constant $\const{8}>0$ such that the $\bar{\mathbb{P}}$-probability that Algorithm~\ref{alg:questAdditiveQuantile} fails to identify the correct topology of the triple $\mathcal{X}$ is bounded from above by  
	$$4 \exp \left(  - \frac{\left| \mquant{2} \right| p(3f/4)^2}{\const{8} \left( p(3f/4) + \alpha \right)} \right),$$
	provided the conditions of Proposition~\ref{prop:quantileTestPop} hold. 
\end{proposition}

\subsection{Proof of Proposition~\ref{prop:quantileBehavior}}
\label{sec:proofOfQuantileBehaviorLemma}
In this section we prove Proposition~\ref{prop:quantileBehavior}, which provides us a control over the behavior of the empirical quantiles computed in the first part of Algorithm~\ref{alg:questAdditiveQuantile}. 

\bigskip

\begin{prevproof}{Proposition}{prop:quantileBehavior}
	Proposition~\ref{prop:reduction-height-diff} guarantees that $\widehat{\Delta}_{xy}$ and $\Delta_{xy}$ are close. Therefore, using the fact that $p(\cdot)$ is a Lipschitz function, we know that there exists a constant $\const{5}'>0$ such that 
	$$\left| r_{xy}^{(0)} - q_{xy}^{(0)} \right|\leq \const{5}' \phi.$$
	The second containment in the statement of the lemma follows from this (after adjusting the constant appropriately). 
	
	We will prove the first part following along the lines of \cite{MosselRoch:15}.
	Let $W$ be the number of genes $i \in \mquant{1}$ that are such that $\widehat{q}_{xy}^i \leq {q}^{(0)}_{xy}$, and let $\widetilde{W}$ be the number of genes $i \in \mquant{1}$ such that $\widehat{q}_{xy}^i \leq q^{(0)}_{xy} + \const{5} \alpha$; we will choose $\const{5}$ below. Notice that the conclusion of Proposition~\ref{prop:quantileBehavior} follows if we show that there is a const $\const{3}>0$ such that $W \leq \left| \mquant{1} \right|\const{3} \alpha$ and $\widetilde{W} \geq \left| \mquant{1} \right| \const{3} \alpha$. So, we bound the probability that each of these events fail. 
	First, we restate the following lemma about the cumulative distribution function from~\cite{MosselRoch:15}.
	\begin{lemma}[CDF behavior~\cite{MosselRoch:15}]
		\label{lem:cdfBehavior}
		There exists a constant $\const{3}'>0$ such that 
		\begin{align*}
		\mathbb{P}\left[ \widehat{q}_{xy} \leq q^{(0)}_{xy} \right] \leq \frac{\const{3}'}{\sqrt{k}}
		\end{align*}	
	\end{lemma}
	\noindent Therefore, the above lemma implies that 
	\begin{align*}
	\mathbb{P}\left[ \widehat{q}_{xy}\leq {q}_{xy}^{(0)} \right] &\leq \frac{\const{3}'}{\sqrt{k}}\leq \const{3}' \alpha.
	\end{align*}
	On the other hand, for every constant $\const{5}''>0$, there is a constant $\const{3}''>0$ such that 
	\begin{align*}
	\mathbb{P}\left[\widehat{q}_{xy}\leq q_{xy}^{(0)} + \const{5}'' \alpha\right] &\geq \mathbb{P}\left[ \widehat{q}_{xy} \leq {q}_{xy}^{(0)} + \const{5}'' \alpha \middle | q_{xy} \in \left[ {q}_{xy}^{(0)}, {q}_{xy}^{(0)} + \const{5}'' \alpha \right]\right]\mathbb{P}\left[q_{xy}\in \left[ {q}_{xy}^{(0)}, {q}_{xy}^{(0)} + \const{5}'' \alpha \right] \right]\nonumber\\
	&\geq \const{3}'' \alpha,
	\end{align*}
	where the last inequality follows from the Berry-Ess\'een theorem (see e.g.~\cite{Durrett:96}) and~\eqref{lem:mscQuantile}. We choose  $\const{5}''$ large enough so that we can take $\const{3}'' > \const{3}'$. Then, we choose $\const{5}$ to be the maximum among $\const{5}'$ and $\const{5}''$; observe that the above inequality holds when $\const{5}''$ is replaced by $\const{5}$. 
	Finally, we set $\const{3} = \frac{\const{3}' + \const{3}''}{2}$.
	
	There is a constant $\const{4}>0$ such that  
	\begin{align}
	\mathbb{P}\left[ W \geq \left| \mquant{1} \right| \const{3} \alpha \right] &= \mathbb{P}\left[ W -  \left| \mquant{1} \right|\mathbb{P}\left[ \widehat{q}_{xy}\leq {q}_{xy}^{(0)} \right]\geq \left| \mquant{1} \right| \left(\const{3}\alpha - \mathbb{P}\left[ \widehat{q}_{xy}\leq  {q}_{xy}^{(0)} \right]\right)\right]\nonumber\\
	&\leq \mathbb{P}\left[ W -  \left| \mquant{1} \right|\mathbb{P}\left[ \widehat{q}_{xy}\leq  {q}_{xy}^{(0)} \right]\geq \left| \mquant{1} \right| \frac{(\const{3} - \const{3}')}{2}\alpha \right]\nonumber\\
	&\leq \exp\left( -\const{4} \left| \mquant{1} \right|\alpha \right), 
	\end{align}
	where the last step follows from Bernstein's inequality (see e.g.~\cite{boucheron2013concentration}). Similarly, it can be shown that 
	\begin{align}
	\mathbb{P}\left[ \widetilde{W} \leq  \left| \mquant{1} \right| \const{3} \alpha\right] &\leq \exp\left( -\const{4}  \left| \mquant{1} \right| \alpha\right).
	\end{align} 
	Now, a union bound over these two probabilities for each of the three pairs of leaves in $\mathcal{X}$ gives us the stated result. 
\end{prevproof}

\subsection{Proof of Proposition~\ref{prop:quantileTestPop}}
\label{sec:pop-quantile-test-proof}

In this section, we show that the expected version of the quantile test succeeds, and hence prove Proposition~\ref{prop:quantileTestPop}  . 

\bigskip

\begin{prevproof}{Proposition}{prop:quantileTestPop}
	Recall that we fix a triple of leaves $\mathcal{X} = \left\{ 1,2,3 \right\}$ such that their topology with respect to $S$ is given by $12|3$ without loss of generality. 

	Let $\mathcal{E}_{12|3}$ be the event that there is a coalescence in the internal branch and observe that $s_{12}$ can be decomposed as follows 
	\begin{align}
	s_{12} &= \Econdred [\widehat{s}_{12}] = \Pcondred \left[ \widehat{q}_{12} \leq \widehat{q}_\ast \right]\nonumber\\
	& = \Pcondred \left[ \mathcal{E}_{12|3}\right] \Pcondred \left[ \widehat{q}_{12} \leq \widehat{q}_\ast \middle | \mathcal{E}_{12|3}\right] + \Pcondred \left[\mathcal{E}_{12|3}^c  \right]\Pcondred \left[ \widehat{q}_{12} \leq \widehat{q}_\ast \middle | \mathcal{E}_{12|3}^c\right] \label{eq:conditionedOnCoalescence}
	\end{align}
	In the proof in \cite{MosselRoch:15}, instead of $\widehat{q}_{12}$, one deals with $\widehat{r}_{12}$; indeed, in the ultrametric setting, there is no difference correction. And it follows from the symmetries of the MSC (namely, the exchangeability of the lineages in a population) that $\Pcondred \left[ \widehat{r}_{12} \leq \widehat{q}_\ast \middle | \mathcal{E}_{12|3}^c\right] = \Pcondred \left[\widehat{r}_{13} \leq \widehat{q}_\ast  \right]$. This turns out to suffice to establish the expected version of the quantile test in~\cite{MosselRoch:15}. In our setting, however, we must control quantitatively the difference between these two probabilities due to the slack added by the reduction step (Algorithm~\ref{alg:questAdditiveReduction}). 
	\begin{lemma}[Closeness to symmetry]
		\label{prop:roos+msc}
		For $\widehat{q}_\ast$ as defined in Algorithm~\ref{alg:questAdditiveQuantile} and any constant $C_2>0$, there exist constants $\const{7}', \const{7}'' >0$ such that 
		\begin{align*}
		\left| \Pcondred \left[ \widehat{q}_{13} \leq \widehat{q}_\ast \right] - \Pcondred \left[ \widehat{q}_{12} \leq \widehat{q}_\ast \middle | \mathcal{E}_{12|3}^c\right]\right| \leq \phi_2 \triangleq \const{7}' \phi \sqrt{\log k},
		\end{align*}
		provided 
		\begin{align*}
		m &\geq \const{7}''\frac{1}{\phi\sqrt{\log k}} \log\left( \frac{1}{\phi \sqrt{\log k}} \right)\\
		k &\geq \left( \frac{1}{\phi} \right)^{1/C_2}.
		\end{align*}
	\end{lemma}
	\noindent We prove this in Section~\ref{sec:proofOfRoos+MSC}.  
	
	Using the above lemma in~\eqref{eq:conditionedOnCoalescence}, we can now bound $s_{12}$ from below as follows
	\begin{align}
	s_{12} &\geq \Pcondred \left[ \mathcal{E}_{12|3}\right] \Pcondred \left[ \widehat{q}_{12} \leq \widehat{q}_\ast \middle | \mathcal{E}_{12|3}\right] + \Pcondred \left[\mathcal{E}_{12|3}^c  \right]\Pcondred \left[ \widehat{r}_{13} \leq \widehat{q}_\ast \right] - \phi_2\nonumber\\
	& = \Pcondred \left[ \mathcal{E}_{12|3}\right]\left( \Pcondred \left[ \widehat{q}_{12} \leq \widehat{q}_\ast \middle | \mathcal{E}_{12|3}\right] - s_{13}\right)  + s_{13} - \phi_2.\label{eq:pitstop1} 
	\end{align}
	This implies that $$s_{12} - s_{13}> \Pcondred \left[ \mathcal{E}_{12|3}\right]\left( \Pcondred \left[ \widehat{q}_{12} \leq \widehat{q}_\ast \middle | \mathcal{E}_{12|3}\right] - s_{13}\right) - \phi_2.$$ The expected version of the quantile test succeeds provided the latter quantity is bounded from below by $0$. We establish a better lower bound, which will be useful in the analysis of the sample version of the quantile test.  
	Towards this end, we will state the following lemma, which is proved in Section~\ref{sec:proofOfproposition1}. 
	\begin{lemma}[Bounds on tails]
		\label{prop:popLowerBoundCompletion}
		There exist positive constants $\const{6}'$ and $\const{6}''$ such that the following hold
		\begin{align*}
		\Pcondred \left[ \widehat{q}_{12} \leq \widehat{q}_\ast \middle | \mathcal{E}_{12|3}\right] &\geq \const{6}'\\
		s_{13} = \Pcondred \left[ \widehat{q}_{13} \leq \widehat{q}_\ast \right] &\leq \const{6}'' \alpha.
		\end{align*}	
	\end{lemma}
	\noindent The first inequality captures the intuition that, conditioned on the coalescence, the probability of $\widehat{q}_{12}$ being small is high. The second inequality captures the intuition that since $\widehat{q}_\ast$ behaves roughly like  $q_{13}^{(\alpha)} = p(\delta_{13}^{(\alpha)} + \widehat{\Delta}_{13})$, the event that $\widehat{q}_{13}\leq \widehat{q}_\ast$ is dominated by the event that the underlying MSC random variable satisfies the same inequality (the deviations of the JC random variable on top of this being of order $k^{-0.5}$). 
	
	Notice that, if we use Lemma~\ref{prop:popLowerBoundCompletion} in \eqref{eq:pitstop1}, there is a constant $\const{6}>0$
	\begin{align*}
	s_{12} - s_{13} \geq \const{6}\, \Pcondred \left[ \mathcal{E}_{12|3} \right]
	\end{align*}
	provided $\phi_2 \leq \const{6} p(3f/4)$ for a large enough $c_6 > 0$, where we used that $\Pcondred \left[ \mathcal{E}_{12|3} \right]$ is lower bounded by $p(3f/4)$. This, along with a similar argument for $s_{23}$, concludes the proof of  Proposition~\ref{prop:quantileTestPop}. 
\end{prevproof}

\subsubsection{Proof of Lemma~\ref{prop:roos+msc}}
\label{sec:proofOfRoos+MSC}

In this section, we prove Lemma~\ref{prop:roos+msc} which is key to accounting for the slack added in the reduction phase of Algorithm~\ref{alg:questAdditiveReduction}. 

\bigskip

\begin{prevproof}{Lemma}{prop:roos+msc}
	First observe that $\Pcondred \left[ \widehat{r}_{13} \leq \widehat{q}^\ast \middle | \mathcal{E}_{12|3}^c \right] = \Pcondred \left[ \widehat{r}_{13} \leq \widehat{q}^\ast \right]$. We prove Lemma~\ref{prop:roos+msc} by arguing that 
	 $\Pcondred \left[ \widehat{q}_{12} \leq \widehat{q}^\ast \middle | \mathcal{E}_{12|3}^c \right]$ is close to $\Pcondred \left[ \widehat{r}_{13} \leq \widehat{q}^\ast \middle | \mathcal{E}_{12|3}^c \right]$, and that
	 $\Pcondred \left[  \widehat{q}_{13} \leq \widehat{q}^\ast  \right]$  is close to $\Pcondred \left[ \widehat{r}_{13} \leq \widehat{q}^\ast \right]$. Both these statements follow from Lemma~\ref{lem:roos+msc} below.
	
	For a pair of leaves $x, y \in \mathcal{X}$, let $\delta_{xy}$ is the distance between $x$ and $y$ on a random gene tree drawn according to the MSC, and let $p_{xy} = p(\delta_{xy})$ denote the corresponding expected $p$-distance. \textit{Conditioned on the value of $\delta_{xy}$}, suppose that we have two Bernoulli random variables $J_1\sim {\rm Bin}(k, p_{xy})$ and $J_2\sim {\rm Bin}(k,p_{xy} + \beta)$, for some fixed $\beta\in (0, 1-p_{xy})$. Then we have the following. 
	
	\begin{lemma}[Mixture of binomials: CDF perturbation]
		\label{lem:roos+msc}
		Suppose that we are given constants $\const{14},\gamma>0$ such that $\gamma < p_{xy}^{(\const{14} \alpha)}$. Then, for any constant $C_2>0$ there exist  constants $\const{13},\const{13}' >0$ such that the following holds
		\begin{align*}
		\left| \Pcondred \left[ J_1 \leq k \gamma \right] - \Pcondred \left[ J_2 \leq k \gamma \right]\right| \leq \const{13}' \beta \sqrt{\log k},
		\end{align*}
		provided 
		\begin{align*}
		m &\geq \const{13}\frac{1}{\beta\sqrt{\log k}} \log\left( \frac{1}{\beta \sqrt{\log k}} \right)\\
		k &\geq \left( \frac{1}{\beta} \right)^{1/C_2}.
		\end{align*}
		
	\end{lemma}
	\noindent Observe that, although $J_1$ and $J_2$ above do not depend on $m$, $\gamma$---through $\alpha$---does.
	While we stated the lemma in terms of $p_{xy}$, this lemma applies to Farris-transformed variables $q_{xy}$ and $r_{xy}$ as well. We prove this lemma at the end of this section. Notice, first, that this result implies that there exists a constant $\const{7}'>0$ such that 
	\begin{align*}
	\left| \Pcondred \left[ \widehat{r}_{13} \leq \widehat{q}^\ast \right] - \Pcondred \left[  \widehat{q}_{13} \leq \widehat{q}^\ast  \right] \right| &\leq \frac{\const{7}'}{2}\phi\sqrt{\log k}\nonumber\\
	\left| \Pcondred \left[ \widehat{q}_{12} \leq \widehat{q}^\ast \middle | \mathcal{E}_{12|3}^c \right] - \Pcondred \left[ \widehat{r}_{13} \leq \widehat{q}^\ast \middle | \mathcal{E}_{12|3}^c \right]\right| &\leq \frac{\const{7}'}{2}\phi\sqrt{\log k}\nonumber.
	\end{align*}
	To see why this is true, first observe that $\widehat{q}^\ast \leq q^{(0)}_{13} + \const{5} \alpha \leq q_{13}^{(\const{5}' \alpha)}$; the first inequality follows from Proposition~\ref{prop:quantileBehavior}, and the second inequality follows from~\eqref{lem:mscQuantile}. This is also true (up to a factor of $\phi$) if we replace the r.h.s.~random variables by $r$s. So we can take $\gamma = \widehat{q}^\ast$. Using this, and taking $\beta$ to be $\mathcal{O}(\phi)$), we get the above two inequalities. This concludes the proof of Lemma~\ref{prop:roos+msc}. 
\end{prevproof}

\begin{prevproof}{Lemma}{lem:roos+msc}
All that remains is to prove  Lemma~\ref{lem:roos+msc}. Before doing this, we prove an auxiliary lemma which characterizes the difference between two binomial distributions in terms of the difference of the underlying probabilities, i.e., we condition on $p_{xy}$. This follows from the work of Roos~\cite{roos2001binomial}. For $J_1$ and $J_2$ as defined above: 
\begin{lemma}[Binomial: CDF perturbation]
	\label{lem:roos}
	For any $\gamma\in (0,1)$, we have 
	\begin{equation*}
	\left|\Pcondred \left[ J_1 \leq k \gamma\middle | p_{xy}\right] - \Pcondred \left[ J_2 \leq k \gamma\middle | p_{xy}\right]\right| \leq \frac{2\sqrt{2} e \sqrt{k+2}}{\sqrt{ (p_{xy}+\beta) (1 -p_{xy}-\beta)}} \beta.
	\end{equation*}
\end{lemma}
\begin{proof}
	It holds that 
	\begin{align*}
	\left|\Pcondred \left[ J_1 \leq k \gamma\middle | p_{xy}\right] - \Pcondred \left[ J_2 \leq k \gamma\middle | p_{xy}\right]\right|
	&\leq \left\| {\rm Bin}(k,p_{xy}) - {\rm Bin}(k,p_{xy}+\beta) \right\|_1\\
	&\leq \sqrt{e} \frac{\sqrt{\theta(p_{xy}+\beta)}}{\left( 1 - \sqrt{\theta(p_{xy}+\beta)} \right)^2}, \;\; \mbox{if }\theta(p_{xy}+\beta) < 1
	\end{align*}
	where $\theta(p_{xy}+\beta) = \frac{\beta^2 (k+2)}{2 (p_{xy}+\beta) (1 - p_{xy}-\beta)}$, and the above inequality comes from \cite[(15)]{roos2001binomial},  by setting $s = 0$ there, and choosing the Poisson-Binomial distribution to simply be the binomial distribution Bin$(k,p_{xy})$. If $ \beta  \leq \sqrt{\frac{(p_{xy}+\beta) (1 -p_{xy}-\beta)}{2 (k+2)}}$, then $1 - \sqrt{\theta(p_{xy}+\beta)} \geq 0.5$. In this case, we have 
	\begin{align*}
	\left| \Pcondred \left[ J_1 \leq k \gamma\middle | p_{xy}\right] - \Pcondred \left[ J_2 \leq k \gamma\middle | p_{xy}\right]\right| \leq   \frac{2\sqrt{2} e \sqrt{k+2}}{\sqrt{ (p_{xy}+\beta) (1 -p_{xy}-\beta)}}\beta.
	\end{align*}
	On the other hand, if $\beta > \sqrt{\frac{(p_{xy}+\beta) (1 - p_{xy}-\beta)}{2 (k+2)}}$, then since the difference between two probabilities is upper bounded by $2$, the following upper bound holds trivially 
	\begin{align*}
	\left|\Pcondred \left[ J_1 \leq k \gamma\middle | p_{xy}\right] - \Pcondred \left[ J_2 \leq k \gamma\middle | p_{xy}\right]\right| \leq \frac{2\sqrt{2} \sqrt{k+2}}{\sqrt{(p_{xy}+\beta) (1 - p_{xy}+\beta)}} \beta.
	\end{align*}
	This concludes the proof. 
\end{proof}
We cannot directly apply Lemma~\ref{lem:roos} to prove Lemma~\ref{lem:roos+msc} since the $\sqrt{k+2}$ factor on the upper bound is too loose for our purposes. Instead, we employ a more careful argument that splits the domain of the underlying MSC random variables.

	First, recall that $\alpha = \max\left\{ \sqrt{\frac{\log k}{k}}, \frac{\log m}{m}\right\}$. Now, consider the following partition of the domain of $p_{xy}$; we will choose $\constf$ below.  
	\begin{align*}
	I_1 &= \left[ p_{xy}^{(0)}, p_{xy}^{(2\const{14} \sqrt{\frac{\log k}{k}})} \right],\; \mbox{\bf low substitution regime for small $k$}\\
	I_2 &= \left[p_{xy}^{(2\const{14} \sqrt{\frac{\log k}{k}})}, p_{xy}^{(2\const{14} \alpha)}\right], \; \mbox{\bf low substitution regime for large $k$; (empty if $\alpha = \sqrt{\frac{\log k}{k}}$)}\\
	I_3&= \left[ p_{xy}^{(2\const{14} \alpha)},0.5\right],\; \mbox{\bf high substitution regime}
	\end{align*}
	Now, we proceed by bounding the above difference in each of these intervals. \\

	\noindent\underline{\bf Low substitution regime, small $k$ ($p\in I_1$)}\\
	In this case, we use the fact that Lemma~\ref{lem:roos} guarantees that the binomial distributions are $\mathcal{O}(\beta\sqrt{k})$ apart. That is, there exists a constant $\const{15}' > 0$ such that 
	\begin{align}
	\Pcondred \left[ p_{xy}\in I_1 \right]\Econdred \left[\left| \Pcondred \left[ J_1 \leq k \gamma \right] - \Pcondred \left[ J_2 \leq k \gamma \right]\right|\middle | p_{xy}\in I_1\right]&\leq \Pcondred \left[ p_{xy}\in I_1 \right] \const{15}' \beta \sqrt{k+2}\nonumber\\
	&\leq \const{15}' \beta\sqrt{\log k},\nonumber
	\end{align}
	where the last step follows after appropriately increasing the constant $\const{15}'$. This follows from the definition of a quantile. 
	
	\bigskip
	
	\noindent\underline{\bf Low substitution regime, large $k$ ($p\in I_2$)}\\
	Notice that if $\alpha = \sqrt{\frac{\log k}{k}}$, then this interval is empty. In the case that it is not, there exists a constant $\const{15}'' >0$ such that 
	\begin{align}
	\Pcondred \left[ p_{xy}\in I_2 \right]\mathbb{E}\left[\left| \Pcondred \left[ J_1 \leq k\gamma \right] - \Pcondred \left[ J_2 \leq k\gamma \right]\right|\middle | p_{xy}\in I_2\right]&\leq \Pcondred [p_{xy}\in I_2]\nonumber\\
	&\stackrel{(a)}{\leq} \const{15}'' \alpha \stackrel{(b)}{\leq} \const{15}''\frac{\log m}{m}.\nonumber
	\end{align}
	$(a)$ follows from~\eqref{lem:mscQuantile}, and $(b)$ follows from the definition of $\alpha$.

	\bigskip
	
	\noindent\underline{\bf High substitution regime ($p\in I_3$)}\\
	In this case observe that, since $\gamma < p_{xy}^{(\const{14} \alpha)}$ (i.e., we are looking at a left tail below the mean), we can apply Chernoff's bound (see e.g.~\cite{MotwaniRaghavan:95}) on each of the two terms in the difference individually. In fact, depending on how large we want $C_2$ to be, we can choose $\const{14}>0$ so that the following inequality holds: 
	\begin{align}
	\Pcondred \left[ p_{xy}\in I_3 \right]\mathbb{E}\left[\left| \Pcondred \left[  J_1 \leq k \gamma \right] - \Pcondred \left[  J_2 \leq k \gamma \right]\right|\middle | p_{xy}\in I_3\right]&\leq \const{15}''' k^{-C_2},\nonumber
	\end{align}
	for some $\const{15}'''>0$. 
	
	\ \\ 
	
	Putting the bounds in the above three regimes together, we see that there is a constant $\const{7}'>0$ (that does not depend on $f, m, k$) such that 
	\begin{align*}
	\left| \Pcondred \left[ J_1 \leq k \gamma \right] - \Pcondred \left[ J_2 \leq k \gamma \right]\right| \leq \frac{\const{7}'}{3} \left( \beta \sqrt{\log k} + \frac{\log m}{m} + k^{-C_2} \right).
	\end{align*}
	The lemma follows by observing that there exists a constant $\const{7}''>0$ such that the following inequalities respectively imply that $m^{-1}\log m \leq \beta \sqrt{\log k}$, and that $k^{-C_2} \leq \beta\sqrt{\log k}$. 
	\begin{align*}
	m &\geq \const{7}\frac{1}{\beta\sqrt{\log k}} \log\left( \frac{1}{\beta \sqrt{\log k}} \right),\\
	k &\geq \left( \frac{1}{\beta} \right)^{1/C_2}.
	\end{align*}
	This concludes the proof of Lemma~\ref{lem:roos+msc}. 
\end{prevproof}

\subsubsection{Proof of Lemma~\ref{prop:popLowerBoundCompletion}}
\label{sec:proofOfproposition1}
In this section, we prove Lemma~\ref{prop:popLowerBoundCompletion}, which is the final piece needed to complete the proof of Proposition~\ref{prop:quantileTestPop}.

\begin{prevproof}{Lemma}{prop:popLowerBoundCompletion}
	Notice that, from Proposition~\ref{prop:quantileBehavior} (on which $\bar{\mathbb{P}}$ is conditioning), we know that $\widehat{q}_\ast \leq r_{13}^{(0)} + \const{5} \phi + \const{5} \alpha$. This implies the second inequality of the lemma, i.e., there exists a constant $\const{6}''>0$ such that $\Pcondred \left[ \widehat{r}_{13} \leq \widehat{q}_\ast \right]\leq \const{6}'' \alpha$ by~\eqref{lem:mscQuantile}.

	To see the first implication of the lemma, we reason as follows. First we make a few observations: 
	\begin{enumerate}
		\item Again, from Proposition~\ref{prop:quantileBehavior}, $\widehat{q}_\ast \geq \widehat{q}_{13}^{(\const{3} \alpha)}\geq {r}_{13}^{(0)}-\const{5}\phi$; the last equality follows from the definitions. 
		\item By definition $q_{12} = p(\delta_{12} + \widehat{\Delta}_{12})$ and, conditioned on $\delta_{12}$, $\widehat{q}_{12}$ is distributed as Bin$(k,q_{12})$.
		\item Because $q_{12} = p(\delta_{12} + \widehat{\Delta}_{12})$,
		$r_{12} = p(\delta_{12} + \Delta_{12})$ and by Proposition~\ref{prop:reduction-height-diff}, it follows that there is $c_{16} > 0$ such that the event $\{r_{12} \leq {r}_{13}^{(0)}-\const{5}\phi - \const{16} \phi\}$ implies the event $\{q_{12} \leq {r}_{13}^{(0)}-\const{5}\phi\}$ under $	\Pcondred$.
		\item The event $\mathcal{E}_{12|3}$ is equivalent to the condition that $r_{12} = p(\delta_{12} + \Delta_{12}) \leq p(\mu_{13} + \Delta_{13}) = r_{13}^{(0)}$.  
	\end{enumerate} 
	We use these facts in the following chain of inequalities, which will lead us to a lower bound on the desired quantity:
	\begin{align}
	\Pcondred \left[ \widehat{q}_{12} \leq \widehat{q}_\ast \middle | \mathcal{E}_{12|3}\right]&\stackrel{(a)}{\geq} \Pcondred \left[ \widehat{q}_{12} \leq {r}_{13}^{(0)}-\const{5}\phi \middle | \mathcal{E}_{12|3}\right]\nonumber\\
	&\stackrel{(b)}{\geq} \Pcondred \left[ {\rm Bin}(k, q_{12}) \leq {r}_{13}^{(0)}-\const{5}\phi\middle | q_{12} \leq {r}_{13}^{(0)}-\const{5}\phi, \mathcal{E}_{12|3} \right]\nonumber\\
	&\qquad \times \Pcondred \left[ q_{12} \leq {r}_{13}^{(0)}-\const{5}\phi \middle |  \mathcal{E}_{12|3}\right]\nonumber\\
	&\stackrel{(c)}{\geq} \Pcondred \left[ {\rm Bin}(k, q_{12}) \leq {r}_{13}^{(0)}-\const{5}\phi\middle | q_{12} \leq {r}_{13}^{(0)}-\const{5}\phi, \mathcal{E}_{12|3} \right]\nonumber\\
&\qquad \times \Pcondred \left[ r_{12} \leq {r}_{13}^{(0)}-\const{5}\phi - \const{16} \phi \middle |  r_{12} \leq r_{13}^{(0)} \right]\nonumber\\	
&\stackrel{(d)}{\geq} \constj' \nonumber,
	\end{align}
	for some constant $C'_j > 0$,
	where $(a)$ follows from Observation 1 above. Inequality $(b)$ follows after conditioning on the event that $q_{12} \leq {r}_{13}^{(0)}-\const{5}\phi$, and $(c)$ follows from Observations 3 and 4 above. Finally, the inequality $(d)$ follows from the Berry-Ess\'een theorem (see e.g.~\cite{Durrett:96}), which gives a constant lower bound on the first line, and~\eqref{lem:mscQuantile} together with the assumption that $\phi \ll p(3f/4)$ and the fact that
	the probability of $\mathcal{E}_{12|3}$ is lower bounded by $p(3f/4)$, which gives a constant lower bound on the second line.
\end{prevproof}


\subsection{Proof of Proposition~\ref{prop:sample-quantile-test}}
\label{sec:proof-sample-quantile-test}

We first prove the following lemma: 
\begin{lemma}[Variance upper bound] 
	\label{lem:quantileTest-varianceUpperBound}
	There is a constant $\const{8}'>0$ such that for $x,y \in \xcal$,
 $$\overline{{\rm Var}}(\widehat{s}_{xy}) \leq \frac{\const{8}'}{\left| \mquant{2} \right|} \left( \bar{\mathbb{P}}\left[ \mathcal{E}_{12|3} \right] + \phi + \alpha\right)$$
	where recall that $\mathcal{E}_{12|3}$ is the event that there is a coalescence in the internal branch.
\end{lemma}
\begin{prevproof}{Lemma}{lem:quantileTest-varianceUpperBound}
	We begin by observing that the variance of $\widehat{s}_{12}$ is bounded from above by $\frac{1}{\left| \mquant{2} \right|}\Pcondred \left[ \widehat{q}_{12}\leq \widehat{q}_\ast \right]$. So we devote the rest of the proof to controlling this probability. 
	
	First, observe that by Proposition~\ref{prop:quantileBehavior}, $\widehat{q}_\ast\leq r_{13}^{(0)} + \const{5}\phi + \const{5} \alpha$.  Also, observe that conditioned on the random distance $\delta_{12}$, $\widehat{q}_{12}$ is distributed according to Bin$(k, q_{12})$, where $q_{12} = p\left( \delta_{12} + \widehat{\Delta}_{12} \right)$. Finally, we observe that from Lemma~\ref{lem:cdfBehavior},  we can conclude that 
	$$\Pcondred \left[ \widehat{q}_{12} \leq r_{13}^{(0)} + \const{5}\phi + \const{5} \alpha \middle | q_{12} > r_{13}^{(0)} + \const{5}\phi + \const{5}\alpha\right] \leq \const{3}' \alpha.$$
	Note that while this looks different from what Lemma~\ref{lem:cdfBehavior} guarantees, it follows given the memoryless property of the exponential distribution. 
	Therefore, we can bound $\Pcondred \left[ \widehat{q}_{12} \leq \widehat{q}_\ast \right]$ as follows 
	\begin{align}
	\Pcondred \left[ \widehat{q}_{12}\leq \widehat{q}_\ast \right]& \leq \Pcondred \left[ \widehat{q}_{12} \leq r_{13}^{(0)} + \const{5}\phi + \const{5} \alpha \right]\nonumber\\
	&\leq \Pcondred \left[ q_{12} \leq r_{13}^{(0)} + \const{5}\phi + \const{5} \alpha  \right] + \nonumber\\
	&\qquad\Pcondred \left[ \widehat{q}_{12} \leq r_{13}^{(0)} + \const{5}\phi + \const{5} \alpha \middle | q_{12} > r_{13}^{(0)} + \const{5}\phi + \const{5} \alpha \right]\nonumber\\
	&\leq \Pcondred \left[ q_{12} \leq r_{13}^{(0)} + \const{5}\phi + \const{5} \alpha  \right] + \const{3}' \alpha. \label{eq:varUpperBound1}
	\end{align}
	
	From~\eqref{lem:mscQuantile}, it follows that there is a constant $\const{8}''>0$ such that 
	\begin{equation}
	\Pcondred \left[q_{12} \leq r_{13}^{(0)} + \const{5}\phi + \const{5} \alpha   \right]\leq \const{8}'' \left(r_{13}^{(0)} + \const{5}\phi + \const{5} \alpha  - q_{12}^{(0)}\right).	\label{eq:varUpperBound2}
	\end{equation}
	Moreover,
	\begin{align*}
	r_{13}^{(0)} - q_{12}^{(0)} &= p\left( \delta_{13}^{(0)} + \Delta_{13} \right) - p\left( \delta_{12}^{(0)} + \widehat{\Delta}_{12} \right)\\
	&\leq  p\left( \delta_{13}^{(0)} + \Delta_{13} \right) - p\left( \delta_{12}^{(0)} + {\Delta}_{12} \right) + \const{8}''' \phi\\
	&\leq \const{8}''' \left(\Pcondred \left[ \mathcal{E}_{12|3} \right] + \phi\right).
	\end{align*}
	The last inequality follows from the fact that $\Pcondred \left[ \mathcal{E}_{12|3} \right]$ is of the order of the length of the internal branch; and so is the difference on the second line.
	Notice that this along with \eqref{eq:varUpperBound1} and \eqref{eq:varUpperBound2} imply that there is a constant $\const{8}'>0$ (after changing it appropriately) such that $\Pcondred \left[ \widehat{q}_{12} \leq \widehat{q}_\ast \right]\leq \const{8}'\left( \Pcondred \left[ \mathcal{E}_{12|3} \right] + \phi + \alpha \right)$, and concludes the proof.
	\end{prevproof}

\begin{prevproof}{Proposition}{prop:sample-quantile-test}
	Recall that we are restricting our attention to a particular triple of leaves $\mathcal{X} = \left\{1,2,3 \right\}$ which has a topology  $12|3$ with respect to the true species tree $S$. In this case, we know that an error in the quantile test implies that either $\widehat{s}_{13} > \widehat{s}_{12}$ or $\widehat{s}_{23} > \widehat{s}_{12}$. Therefore, we can control the probability that the algorithm makes an error in correctly identifying the topology of the triple at hand as follows. 
	\begin{align}
	\mathbb{P}\left[ {\rm error} \right]&\leq \Pcondred \left[ \widehat{s}_{13} \geq \widehat{s}_{12} \right] + \Pcondred \left[ \widehat{s}_{23} \geq \widehat{s}_{12} \right]\nonumber\\
	&\leq \Pcondred \left[ \widehat{s}_{13} - s_{13} \geq \frac{s_{12} - s_{13}}{2} \right] + \Pcondred \left[ s_{12} - \widehat{s}_{12} \geq \frac{s_{12} - s_{13}}{2} \right]\nonumber\\
	&\qquad + \Pcondred \left[ \widehat{s}_{23} - s_{23} \geq \frac{s_{12} - s_{23}}{2} \right] + \Pcondred \left[ s_{12} - \widehat{s}_{12} \geq \frac{s_{12} - s_{23}}{2} \right]\nonumber\\
	&\triangleq I_1 + I_2 + I_3 + I_4\label{eq:quantileTest-sampleVersion-unionBound}
	\end{align}	
	We will now use  concentration inequalities to control each of the above terms.
	
	Consider $I_2$ first. We need two ingredients for invoking Bernstein's inequality: (1) an effective lower bound on the ``gap'' $\frac{s_{12} - s_{13}}{2}$, and (2) an effective upper bound on the variance of the random variable $\widehat{s}_{12}$. We will use Proposition~\ref{prop:quantileTestPop} for (1), and Lemma~\ref{lem:quantileTest-varianceUpperBound} for (2). That is,
	\begin{align}
	I_2 = \Pcondred \left[ s_{12} - \widehat{s}_{12} \geq \frac{s_{12} - s_{13}}{2} \right] &\stackrel{(a)}{\leq} \exp\left( - \frac{0.5 \left( \frac{s_{12}-s_{13}}{2} \right)^2}{{\rm Var}(\widehat{s}_{12})  + \frac{1}{6}(s_{12}-s_{13})} \right)\nonumber\\
	&\stackrel{(b)}{\leq} \exp\left( - \frac{\left| \mquant{2} \right| \left(\const{6} \Pcondred \left[ \mathcal{E}_{12|3}\right]\right)^2}{\const{8}'\left( \alpha+ \phi +  \Pcondred \left[ \mathcal{E}_{12|3} \right] \right) + \frac{\const{6}}{6}\Pcondred \left[ \mathcal{E}_{12|3} \right]} \right)\nonumber\\
	&\stackrel{(c)}{\leq} \exp\left( - \frac{\left| \mquant{2}˜ \right| p(3f/4)^2}{\const{8}\left( p(3f/4) + \alpha \right)} \right),
	\end{align}
	where $(a)$ follows from Bernstein's inequality (see e.g.~\cite{boucheron2013concentration}); $(b)$ follows from the lower bound on $s_{12} - s_{13}$ provided by Proposition~\ref{prop:quantileTestPop} and the upper bound on Var$(\widehat{s}_{12})$ provided by Lemma~\ref{lem:quantileTest-varianceUpperBound}; and $(c)$ follows from the fact that $\Pcondred \left[ \mathcal{E}_{12|3} \right]$ is bounded from below by $p(3f/4)$. We have absorbed all constants into $\const{8}$. 
	
	We can similarly control $I_1, I_3, $and $I_4$. Putting all this back in \eqref{eq:quantileTest-sampleVersion-unionBound} concludes the proof. 
\end{prevproof}

\end{document}